\documentclass{article}
\usepackage{enumerate}
\usepackage[OT1]{fontenc}
\usepackage{amsmath,amssymb}

\usepackage{smile}
\usepackage{multirow}
\usepackage[colorlinks,
            linkcolor=red,
            anchorcolor=blue,
            citecolor=blue
            ]{hyperref}

\def\given{\,|\,}

\def\tr{\mathop{\text{tr}}\kern.2ex}

\long\def\comment#1{}

\def\tr{\mathop{\text{Tr}}}

\def\cS{{\mathcal{S}}}

\newcommand{\bel}{\begin{eqnarray}\label}
\newcommand{\eel}{\end{eqnarray}}
\newcommand{\bes}{\begin{eqnarray*}}
\newcommand{\ees}{\end{eqnarray*}}

\usepackage{mathrsfs}

\usepackage{fullpage}

\def\##1\#{\begin{align}#1\end{align}}
\def\$#1\${\begin{align*}#1\end{align*}}
\usepackage{enumitem}

\title{\huge  Breaking the Curse of Many Agents:\\Provable Mean Embedding $Q$-Iteration for Mean-Field Reinforcement Learning}

\author
{
\normalsize Lingxiao Wang\thanks{Northwestern University; \texttt{lwang@u.northwestern.edu}}
\qquad
\normalsize Zhuoran Yang\thanks{Princeton University; \texttt{zy6@princeton.edu}}
\qquad
\normalsize Zhaoran Wang\thanks{Northwestern University; \texttt{zhaoranwang@gmail.com}}
}
\date{\today}

\begin{document}
\maketitle

\begin{abstract}
Multi-agent reinforcement learning (MARL) achieves significant empirical successes. However, MARL suffers from the curse of many agents. In this paper, we exploit the symmetry of agents in MARL. In the most generic form, we study a mean-field MARL problem. Such a mean-field MARL is defined on mean-field states, which are distributions that are supported on continuous space. Based on the mean embedding of the distributions, we propose MF-FQI algorithm that solves the mean-field MARL and establishes a non-asymptotic analysis for MF-FQI algorithm. We highlight that MF-FQI algorithm enjoys a ``blessing of many agents'' property in the sense that a larger number of observed agents improves the performance of MF-FQI algorithm.
\end{abstract}


\section{Introduction}\label{sec::intro}

Reinforcement learning (RL) \citep{sutton2018reinforcement} searches for the optimal policy for sequential decision making through interacting with environments and learning from experiences. Multi-agent reinforcement learning (MARL) \citep{bu2008comprehensive} generalizes   RL to   multi-agent systems. For competitive tasks such as zero-sum game and general-sum game, various  MARL  algorithms \citep{littman1994markov, hu2003nash, wang2003reinforcement} are proposed  in search   for the  Nash equilibrium \citep{nash1951non}. Meanwhile, for cooperative tasks, MARL searches for the optimal policy that maximizes the social welfare \citep{ng1975bentham}, i.e., the expected total reward obtained by all agents \citep{tan1993multi, panait2005cooperative, wang2003reinforcement, claus1998dynamics, lauer2000algorithm, dvzeroski2001relational, guestrin2002coordinated, kar2013cal, zambaldi2018relational}. 
 Combined  with the breakthrough in deep learning, MARL achieves significant empirical successes  in both settings, e.g.,  
autonomous driving \citep{shalev2016safe}, Go \citep{silver2016mastering, silver2017mastering},  esports \citep{alphastarblog, OpenAI_dota}, and robotics \citep{yang2004multiagent}. 

Despite its empirical successes, MARL remains challenging  in the ``many-agent'' setting, as the capacity of state-action space grows exponentially in the number of agents, which hinders the learning of value function and policy due to the curse of dimensionality. Such a challenge is named as the ``curse of many agents''. One way to break such a curse is through mean-field approximation, which exploits the symmetry of homogeneous agents and summarizes them as a population. In the most general form, such a population is represented by a distribution over the state space of individual agents, while the reward and transition are parametrized as functionals of distributions \citep{acciaio2018extended}. Although mean-field MARL demonstrates remarkable efficiency in applications such as large-scale fleet management \citep{lin2018efficient} and ridesharing order dispaching \citep{li2019efficient}, its theoretical analysis remains scarce. In particular, despite significant progress \citep{yang2017learning, yang2018mean, jiang2018learning, guo2019learning}, we still lack a principled model-free algorithm that allows individual agents to have continuous states, which requires approximating nonlinear functionals of infinite-dimensional mean-field states, e.g., value function and policy. 


In this paper, we study mean-field MARL in the collaborative setting, where the mean-field states are distributions over a continuous space $\cS$. Here $\cS$ denotes the state space of individual agents. In particular, we consider the setting with a centralized controller, which has a finite action space $\cA$. Such a setting is extensively studied in the analysis of societal-scale systems \citep{gueant2011mean, gomes2015economic, moll2019uneven}. As a simplified example, the central bank or the central government decides whether to raise the interest rate or reduce the fiscal budget, respectively, both with the goal of maximizing social welfare. In such an example, the action space only contains two actions. However, the action taken by the centralized controller affects the dynamics of billions of individuals. Such a setting can be viewed as centralized mean-field control \citep{huang2012social, carmona2013control, fornasier2014mean} with an infinite number of homogeneous agents, which faces two challenges: (i) learning the value function and policy is intractable as they are functionals of distributions, which are infinite dimensional as $\cS$ is continuous, and (ii) the mean-field state is only accessible through the observation of a finite number of agents, which only provides partial information.~To tackle these challenges, we resort to the mean embedding of mean-field states into a reproducing kernel Hilbert space (RKHS) \citep{gretton2007kernel, smola2007hilbert, sriperumbudur2010hilbert}, which allows us to parametrize value functions as nonlinear functionals over the RKHS. Based on value function approximation, we propose the mean-field fitted Q-iteration algorithm (MF-FQI), which provably attains the optimal value function at a linear rate of convergence. In particular, we show that MF-FQI breaks the curse of many agents in the sense that its computational complexity only scales linearly in the number of observed agents, and moreover, the statistical accuracy enjoys a ``blessing of many agents'', that is, a larger number of observed agents improves the statistical accuracy. Moreover, we characterize the phase transition in the statistical accuracy in terms of the batch size in fitted Q-iteration and the number of observed agents. 


\vskip4pt
\noindent{\bf Our Contribution.} Our contribution is three-fold: (i) We propose the first model-free mean-field MARL algorithm, namely MF-FQI, that allows for continuous support with provable guarantees.~(ii) We prove that MF-FQI breaks the curse of many agents by establishing its nonasymptotic computational and statistical rates of convergence. (iii) We motivate a principled framework for exploiting the invariance in MARL, e.g., exchangeability, via mean embedding.

\vskip4pt
\noindent{\bf Related Works.} Our work is related to mean-field games and mean-field control. The study of mean-field games focuses on the search of the Nash equilibrium \citep{huang2003individual, lasry2006jeux, lasry2006jeux1, lasry2007mean, huang2007large, gueant2011mean, carmona2018probabilistic}, whereas the goal of mean-field control is to  optimally control a McKean-Vlasov process \citep{buckdahn2009mean, buckdahn2011general, andersson2011maximum,meyer2012mean, bensoussan2013mean, carmona2015forward, acciaio2018extended}. Most of these works focus on the continuous-time setting and require the knowledge of the transition model. In contrast, we consider the model-free and discrete-time setting.

Our work falls in the study of mean-field MARL, which generalizes finite-agent MARL by incorporating the notion of mean fields. Previous works investigate mean-field MARL in both cooperative and competitive settings \citep{yang2017learning, yang2018mean,jiang2018learning, guo2019learning}. In \cite{yang2017learning}, a similar setting of mean-field MARL is studied, where the mean-field states are supported on a discrete space and the transition is linear in the state and action. In contrast, our work is model-free and allows for continuous support. In \cite{yang2018mean, guo2019learning}, mean-field MARL algorithms are proposed in the competitive setting, which have provable guarantees when the support is discrete. In comparison, we consider the cooperative setting with continuous support and establish nonasymptotic guarantees. 



Our work exploits the exchangeability of agents via mean embedding. See e.g.,  \cite{smola2007hilbert, fukumizu2008kernel, gretton2009covariate, sriperumbudur2010hilbert, gretton2012kernel, tolstikhin2017minimax} and references therein for the study of mean embedding. Our work is closely related to various statistical models that exploit invariance, such as set kernels \citep{haussler1999convolution, gartner2002multi, kondor2003kernel} and deep sets \citep{zaheer2017deep}. We refer to \cite{bloem2019probabilistic} for a detailed survey on learning with invariance.


\vskip4pt
\noindent{\bf Notations.} For a topological space $X$, we denote by $\cC_B(X)$ the set of bounded and continuous real functions on $X$. We denote by $\cM(X)$ the space of all the probability distributions supported on $X$. For $x \in X$, we denote by $\delta_x \in \cM(X)$ the point mass at $x$. For a real-valued function $f$ defined on $X$, we denote by $\|f\|_{p, \nu}$ the $L_p(\nu)$ norm for $p\geq 1$, where $\nu\in\cM(X)$. We write $\|f\|_{\nu}=\|f\|_{2, \nu}$ for notational simplicity.


\section{Mean-Field MARL via Mean Embedding}
\label{sec::background}

In this section, we first motivate mean-field MARL by an example of $N$-player control with invariance. We then introduce the problem setup of mean-field MARL and the mean embedding of distributions. Finally, we propose the MF-FQI algorithm, which solves mean-field MARL based on the mean embedding.

\subsection{Exchangeability in MARL}
\label{sec::n_player_sym}
We first consider an $N$-player control problem with a centralized controller in discrete time. 
Such a setting is extensively studied in the analysis of societal-scale systems \citep{gueant2011mean, gomes2015economic, moll2019uneven}, such as the example of central bank or central government in \S\ref{sec::intro}. At each time step $t$, the central controller takes an action $a_{t}\in\cA$ based on the current joint state $s_t = (s_{1, t}, \ldots, s_{N, t})$, where $s_{i, t} \in \cS$ is the state of the $i$-th agent at time $t$. The immediate reward $r_t$ follows a distribution that depends on the current state $s_t \in \cS^N$ and action $a_t\in\cA$. The transition of the joint state follows a distribution, which is determined by the current state $s_t \in \cS^N$ and action $a_t\in \cA$. In summary, it holds that
\#\label{eq::n_player_mdp_invar}
r_t \sim r(s_t, a_t), \quad S_{t+1} \sim P(\cdot~|~s_{t}, a_t).
\#
The process defined by the tuple $(\cS, \cA, P, r)$ is a Markov decision process (MDP). We define a policy $\pi:\cS^N \mapsto \cM(\cA)$ as a mapping that maps a joint state $s \in \cS^N$ to a probability distribution $\pi(\cdot~|~s)$ over $\cA$. We define the value function corresponding to the policy $\pi$ as  
\#\label{eq::def_V_nplayer}
V^\pi(s) = \EE\biggl[\sum^\infty_{t = 0}\gamma^t\cdot r(S_t, A_t)~\biggl|~ S_0 = s\biggr],
\#
where $A_t \sim \pi(\cdot|S_t)$, and $S_{t+1}\sim \PP(\cdot\given S_t, A_t)$, and $\gamma \in (0, 1)$ is the discount factor. Similarly, we define the action-value function corresponding to the policy $\pi$ as follows,
\#\label{eq::def_Q_nplayer}
Q^\pi(s, a) = \EE\biggl[\sum^\infty_{t = 0}\gamma^t\cdot r(S_t, A_t)~\biggl|~ S_0 = s, A_0 = a\biggr],
\#
where $A_t \sim \pi(\cdot\given S_t)$, and $S_{t+1}\sim \PP(\cdot\given S_t, A_t)$. We define the Bellman operator $T^\pi$ as follows,
\#
T^\pi Q(s, a) = \EE\bigl[r(s, a) + \gamma\cdot Q(S', A')\bigr],
\#
where $S'\sim P(\cdot\given  s, a)$ and $A'\sim \pi(\cdot\given s)$. Our goal is to find the optimal policy that maximizes the expected total reward as follows,
\#\label{eq::def_qstar_nplayer}
Q^*(s, a) = \sup_{\pi} Q^\pi(s, a), \quad \forall (s, a)\in\cS\times\cA.
\#
We denote by $\pi^*$ the optimal solution of \eqref{eq::def_qstar_nplayer}. 
It can be shown that $Q^* = Q^{\pi^*}$, and the following Bellman optimality equation holds,
\#\label{eq::def_bellman_optimal}
Q^*(s, a) &= T Q^*(s, a) =\EE\Bigl[r(s, a) + \gamma\cdot\max_{a\in\cA} Q^*(S', a)\Bigr],
\#
where $S' \sim P(\cdot\given s, a)$. Here we call $T$ the Bellman optimality operator. 
\vskip4pt
\noindent{\bf Curse of Many Agents: }
The learning of the optimal action-value function $Q^*$ under the $N$-player setting suffers from the curse of many agents. More specifically, as $N$ increases, the capacity of the joint state space $\cS^N$ grows exponentially in $N$ and incurs intractability in the learning of the action-value function $Q$. To address such a curse, we exploit the exchangeability of the MDP in \eqref{eq::n_player_mdp_invar}. More specifically, we assume that the MDP is exchangeable in the sense that
\#\label{eq::n_player_mdp_invar1}
&r(s_t, a_t) \stackrel{\ud}{=} r\bigl(\sigma(s_t), a_t\bigr), \qquad P(s_{t+1}~|~s_{t}, a_t) = P\bigl(\sigma(s_{t+1})~\bigl|~\sigma(s_{t}), a_t\bigr),
\#
which holds for any $s_t, s_{t+1} \in \cS^N$, $a_t \in \cA$, and $\sigma \in \SSS_N$. Here $\sigma$ is a block-wise permutation of the vector $s_t\in\cS^N$, and $\SSS_N$ is the permutation group of order $N$. Under the exchangeability defined in \eqref{eq::n_player_mdp_invar1}, the following proposition shows that the optimal policy is invariant to permutations of the joint state.
\begin{proposition}[Invariance of $Q^*$]
\label{prop::inv_pi}
If \eqref{eq::n_player_mdp_invar1} holds, then it holds for any $\sigma \in \SSS_N$, $s \in \cS^N$, and $a \in \cA$ that
\#
Q^*(s, a) = Q^*\bigl(\sigma(s), a\bigr).
\#
Moreover, it holds for any $\sigma \in \SSS_n$, $s \in \cS^N$, and $a \in \cA$ that $\pi^*(a\given s) \stackrel{\ud}{=} \pi^*(a\given \sigma(s))$.
\end{proposition} 
\begin{proof}
See \S\ref{sec::pf_inv_pi} for a detailed proof.
\end{proof}
Meanwhile, the following proposition proves that the action-value function $Q^\pi$ is invariant to permutations of the joint states.
\begin{proposition}[Invariant Representation]
\label{thm::inv::rep}
Let $\pi$ be invariant to permutations such that $\pi(\cdot\given s) \stackrel{\ud}{=} \pi(\cdot\given \sigma(s))$.  If \eqref{eq::n_player_mdp_invar1} holds, then it holds for some $g:\cM(\cS)\times \cA \mapsto \RR$ that
\#\label{eq::q_inv_nplayer}
Q^\pi(s, a) = g(M_s, a).
\#
Here $M_s(\cdot)$ is the empirical measure supported on the set $\{s_i\}_{i\in[N]}$ corresponding to $s = (s_1, \ldots, s_N)$, which takes the form of
\#\label{eq::def_emp_s}
M_s= \frac{1}{N}\sum^N_{i = 1}\delta_{s_i},
\#
where recall that $\delta_{s_i}$ is the point mass at $s_i \in\cS$ for all $i\in[N]$.
\end{proposition}
\begin{proof}
See \S\ref{sec::pf_inv_rep} for a detailed proof.
\end{proof}
By Propositions  \ref{prop::inv_pi} and \ref{thm::inv::rep}, the optimal action-value function $Q^*$ is related to the joint state through the empirical state distribution $M_s$ defined in \eqref{eq::def_emp_s}. When the number of agents $N$ goes to infinity, the empirical state distribution $M_s$ converges to a limiting continuous distribution. 
To capture such a limiting dynamics of infinite agents with exchangeability, we define an MDP with $\cM(\cS)$, the space of probability measures supported on $\cS$, as the mean-field state space as follows. 
\vskip4pt
\noindent{\bf Mean-Field MARL. } We define the discounted mean-field MDP by the tuple $(\cM(\cS), \cA, P, r, \gamma)$. Here $\cA$ is the action space and $\cM(\cS)$ is the mean-field state space, which is the space of all the distributions supported on $\cS$. Given a mean-field state $p_s \in \cM(\cS)$ and an action $a\in\cA$, the immediate reward follows the distribution $r(a, p_s)$, where $r: \cA\times\cM(\cS) \mapsto \cM(\RR)$. The Markov kernel $P(\cdot ~|~ \cdot)$ maps the action-state pair $(a, p_s)$ to a distribution on $\cM(\cS)$, which is the distribution of the mean-field state after transition from the action-state pair $(a, p_s)$. For a policy $\pi(\cdot\given p_s)$ that maps from $\cM(\cS)$ to $\cM(\cA)$, we define the action-value function as
\#\label{eq::q_mean_field}
Q^\pi(a, p_s) = \EE\biggl[\sum^\infty_{t = 0}\gamma^t\cdot r(P_{s, t}, A_t)~\biggl|~ P_{s, 0} = p_s, A_0 = a\biggr],
\#
where $A_t \sim \pi(\cdot\given  P_{s, t})$ and $P_{s, t+1} \sim P(\cdot\given a, p_{s, t})$. Correspondingly, we define the Bellman evaluation operator $T^\pi$ as follows,
\#
T^\pi Q(a, p_s) = \EE\bigl[r(a, p_s) + \gamma\cdot Q(A', P_{s'})\bigr],
\#
where $P_{s'} \sim P(\cdot\given a, p_s)$ and $A'\sim \pi(\cdot\given a, p_s)$. We define the optimal action-value function as $Q^* = \sup_{\pi} Q^\pi$. The Bellman optimality equation then takes the form of
\#\label{eq::q_optimal_mean_field}
Q^*(a, p_s) &= T Q^*(a, p_s)=\EE\Bigl[r(a, p_s) + \gamma\cdot\max_{a\in\cA} Q^*(a, P'_s)\Bigr],
\#
where $P_{s'} \sim P(\cdot\given a, p_s)$. We write $\Omega = \cS\times\cA$ and define the space of state-action configurations $\tilde\cM(\Omega)$ as follows,
\#\label{eq::def_cmomega}
&\tilde\cM(\Omega) = \bigl\{\omega_{a, p_s} = \delta_a\times p_s: a \in \cA, p_s \in \cM(\cS)\bigr\}.
\#
Here we denote by $\delta_a \in\cM(\cA)$ the point mass at action $a \in \cA$ and by $\delta_a\times p_s$ the product measure on $\Omega = \cS\times\cA$ induced by $\delta_{a}$ and $p_s$. See \S\ref{sec::topo_str} for the definition of a topological structure that allows us to define distributions on $\tilde\cM(\Omega)$. Note that the transition kernel $P(\cdot|a, p_s)$ equivalently defines a Markov kernel from $\tilde\cM(\Omega)$ to $\cM(\cS)$. With a slight abuse of notations, we denote by $P(\cdot|\omega_{a, p_s})$ such a Markov kernel and do not distinguish between them. Similarly, we denote by $r(\omega)$ and $Q(\omega)$ the immediate reward and action-value function defined on the state-action configuration $\omega\in\tilde\cM(\Omega)$, respectively. We assume that the action set $\cA$ is finite, and the immediate reward is upper bounded by a positive absolute constant $R_{\max}$. It then holds that the action-value functions are upper bounded by $Q_{\max} = R_{\max}/(1 - \gamma)$.

In a multi-agent environment with infinite homogeneous agents and continuous state space $\cS$, we cannot access the mean-field state $p_s$ directly. Instead, we assume that we observe the states of $N$ agents that follows the mean-field state $p_s$. In what follows, we construct an algorithm that solves the mean-field MARL via such a finite observation for each mean-field state. In the sequel, we denote by $\hat Q^\lambda_\kappa$ and $\pi_\kappa$ the outputs of MF-FQI.

We highlight that the mean-field MARL setting faces two challenges: (i) learning the value function and policy is intractable as they are functionals of distributions, which are infinite dimensional as $\cS$ is continuous, and (ii) the mean-field state is only accessible through the observation of a finite number of agents, which only provides partial information. In what follows, we tackle these challenges via mean embedding.

\subsection{Mean Embedding}
\label{sec::mean_embedding}
To learn the optimal action-value function $Q^*$ defined on $\tilde\cM(\Omega)$, which is a space of distributions, we introduce mean embedding, which embeds the space of distributions to a reproducing kernel Hilbert space (RKHS). We denote by $\cH(k)$ the RKHS with reproducing kernel $k:\Omega\times\Omega \mapsto \RR$. For any state-action configuration $\omega \in \tilde\cM(\Omega)$, the mean embedding $\mu_\omega(\cdot)$ of $\omega$ into the RKHS $\cH(k)$ is defined as follows \citep{gretton2007kernel,smola2007hilbert, sriperumbudur2010hilbert},
\#\label{eq::def_mean_emb}
\mu_{\omega}(x) = \int_{\Omega} k(x, t) \ud \omega(t) \in \cH(k).
\#
Let $\cX = \{\mu_\omega: \omega \in \tilde\cM(\Omega)\}\subseteq \cH(k)$. To tackle challenge (i), we introduce another reproducing kernel $K: \cX\times \cX\mapsto \RR$. Such a kernel generates an RKHS $\cH(K)$ that include functions defined on $\cX$. Our idea is then to approximate $Q^*$ using functions in $\cH(K)$. Note that upon a proper selection of kernel $K(\cdot, \cdot)$, the corresponding RKHS $\cH(K)$ captures a rich family of functions defined on $\cX$. As an example, for universal kernels such as the radial basis function (RBF) kernel, the corresponding RKHS is dense in $\cC(\cX)$. To further regulate the behavior of mean embedding and RKHS, we introduce the following regularity conditions on kernels $k(\cdot, \cdot)$ and $K(\cdot, \cdot)$.
\begin{assumption}[Regularity Condition of Kernels]
\label{asu::bound_kernel}
We assume that the kernel $k(\cdot, \cdot)$ and $K(\cdot, \cdot)$ are continuous and bounded as follows,
\#\label{asu::bound_kernel_1}
&k(u, u)\leq \varrho, \quad \forall u \in \Omega,\qquad K(\mu_{\omega}, \mu_{\omega})\leq \varsigma, \quad \forall \omega \in \tilde\cM(\Omega),
\#
where $\varrho$ and $\varsigma$ are positive absolute constants. We assume that $k(\cdot, \cdot)$ is universal and the mean embeddings $\mu_\omega(\cdot)$ are continuous for any $\omega \in \tilde\cM(\Omega)$. We further assume that $K(\cdot, \mu_\omega)$ is H\"older continuous such that for any $x, y \in\tilde\cM(\Omega)$, it holds that
\#
\|K(\cdot, \mu_x) - K(\cdot, \mu_y)\|_{\cH(K)} \leq L\cdot\| \mu_x - \mu_y\|^h_{\cH(k)},
\#
where $L$ and $h$ are positive absolute constants.
\end{assumption}
The assumption on the boundedness of the kernels in \eqref{asu::bound_kernel_1} is a standard assumption in the learning with kernel embedding \citep{caponnetto2007optimal, muandet2012learning, szabo2015two, lin2017distributed}. The universality assumption on $k(\cdot, \cdot)$ ensures that each mean embedding uniquely characterizes a distribution \citep{gretton2007kernel, gretton2012kernel}. The continuous assumption on the embeddings $\mu_\omega$ is a mild regularity condition, which holds if the kernel $k(\cdot, \cdot)$ is universal and the domain $\Omega$ is compact \citep{sriperumbudur2010hilbert}. Meanwhile, the H\"older continuity of $K(\cdot, \cdot)$ in Assumption \ref{asu::bound_kernel} is a mild regularity condition. Such an assumption holds for a rich family of commonly used reproducing kernels, such as the linear kernel $K(\mu_x, \mu_y) = \langle\mu_x, \mu_y\rangle_{\cH(k)}$ and the RBF kernel $K(\mu_x, \mu_y) = \exp(-\|\mu_x - \mu_y\|^2_{\cH(k)}/\sigma^2)$. 

We highlight that the H\"older continuity of $K(\cdot, \cdot)$ allows for an approximation of mean embedding $\mu_p$ based on the empirical approximation $\hat p$ of the distribution $p \in \tilde\cM(\Omega)$ with finite observations, which further allows for an approximation of the action-value function with finite observations, thus tackles the challenge (ii). For an empirical approximation of state-action configuration $\omega_{a, \hat p_s}= \delta_a\times \hat p_s$, where $\hat p_s$ is the empirical distribution of the observed states $\{s_i\}_{i \in N}$, the mean embedding takes the following form,
\#\label{eq::FQI_omega_hat}
\mu_{\omega_{a,\hat p_s}}(\cdot) = \frac{1}{N}\sum^N_{i = 1}k\bigl(\cdot, (a, s_{i})\bigr),
\#
which is invariant to the permutation of states $\{s_i\}_{i\in[N]}$. Such an invariance is also exploited by the neural-network based approach named deep sets \citep{zaheer2017deep}. In what follows, we connect our mean embedding approach to the invariant deep reinforcement learning under the framework of overparametrized two-layer neural networks \citep{zhang2016understanding, jacot2018neural, neyshabur2018towards, arora2019fine}. 

\vskip4pt
\noindent{\bf Connection to Invariant Deep Reinforcement Learning.} In what follows, we assume that $(a, s_i)\in\RR^d$ and write $x_i = (a, s_i)$. We define the feature mappings $\{\phi_j(\cdot)\}_{j \in [m]}$ and $\{\Phi_\ell(\cdot)\}_{\ell \in [M]}$ as follows,
\#\label{eq::def_phi_nn}
&\phi_j(x) = \frac{1}{\sqrt{m}}\cdot b_j\cdot\ind\{w_j^\top x > 0\}\cdot x, \qquad \Phi_r(q) = \frac{1}{\sqrt{\ell}}\cdot b'_r\cdot\ind\{{W'_r}^\top q > 0\}\cdot q,
\#
where $b_j, b'_r\sim \text{Unif}\{-1, 1\}$,  $w_j \sim N(0, I_{d\ell}/(d\ell)) $, and $ W'_r\sim N(0, I_\ell/\ell)$. Correspondingly, we define the kernels $k_{m}(x, y) = \sum^m_{j = 1}\phi_j(x)^\top \phi_j(y)$ and  $K_\ell(p, q) = \sum^\ell_{r = 1}\Phi_r(p)^\top \Phi_r(q)$.

Note that the mean embedding of a point mass $\delta_{x_i}$ by the kernel $k_m$ is an array $\mu_i = [\phi_1(x_i)^\top, \ldots, \phi_m(x_i)^\top]^\top \in \RR^{md}$. For a mean embedding $\mu \in \RR^{md}$, we consider the parametrization of action-value functions $Q(\mu) = \tilde Q(D^\top \mu)$, where $\tilde Q \in \cH(K_\ell)$ and $D = [D_1, \ldots, D_m] \in \RR^{md\times \ell}$ with $D_j \in \RR^{d\times \ell}$ $(j\in[m])$. Let $\mu_p$ be the mean embedding of the empirical measure supported on $\{x_i\}_{i\in[N]}$. It then holds for some $\{\alpha_r\}_{r\in[\ell]} \subseteq \RR^d$ that
\#\label{eq::f}
Q(\mu_p) &= \sum^\ell_{r = 1} \alpha_r^\top \Phi_r( D^\top \mu_p ) = \frac{1}{\sqrt{\ell}}\sum^\ell_{r = 1} b'_r\cdot \ind\{{W'_r}^\top \rho > 0\}\cdot\alpha_r^\top (D^\top \mu_p)= f(D^\top \mu_p).
\#
Note that if $\alpha_r$ is sufficiently close to $W'_r$, then $Q(\mu_p)$ is close to $\tilde f(D^\top \mu_p)$, where
\#\label{eq::tilde_f}
\tilde f(\rho) = \frac{1}{\sqrt{\ell}}\sum^\ell_{r = 1} b'_r\cdot \ind\{{\alpha_r}^\top \rho > 0\}\cdot\alpha_r^\top (\rho), \quad \forall \rho \in \RR^\ell.
\#
Here $\tilde f(\cdot)$ is a two-layer neural network with parameters $b'_r$, $\alpha_r$ $(r\in[\ell])$ and ReLU activation function. Meanwhile, it holds that
\#\label{eq::psi}
D^\top \mu_p &= \frac{1}{N}\sum^N_{i = 1}\frac{1}{\sqrt{m}}\sum^m_{j = 1} b_j\cdot\ind\{{w_j}^\top x_i > 0\}\cdot D_j^\top x_i =\frac{1}{N}\sum^N_{i = 1} \psi(x_i).
\#
Similarly, if $D_j$ is close to $w_j$, then $\psi(x_i)$ is close to $\tilde \psi(x_i)$ defined as follows,
\#\label{eq::tilde_psi}
\tilde \psi(x_i) = \frac{1}{\sqrt{m}} \sum^m_{j = 1}b_j\cdot\ind\{{D_j}^\top x_i > 0\}\cdot D_j^\top x_i, \quad \forall x_i \in \RR^d.
\#
Here $\tilde \phi(x_i)$ is a two-layer neural network with input $x_i$, parameters $b'_j$ $(j \in [m])$ and $D$, and ReLU activation function. Note that for the functions $f$ and $\psi$ defined in \eqref{eq::f} and \eqref{eq::psi} to approximately take the form of neural networks, we requires the parameters $\alpha$ and $D$ to be sufficiently close to $W$ and $w$, respectively. Such a requirement is formally characterized by the study of overparametrized neural networks. More specifically, if the widths $m$ and $\ell$ of neural networks are sufficiently large (which depends on the deviation of parameters $\alpha$ and $D$ from their respective initializations $W$ and $w$), then the functions $f$ and $\psi$ well approximates the neural networks $\tilde f$ and $\tilde \psi$ defined in \eqref{eq::tilde_f} and \eqref{eq::tilde_psi}, respectively.

In conclusion, under the mean embedding with the feature mappings defined in \eqref{eq::def_phi_nn}, the parameterization of action-value function takes the form of $Q(\mu_p) = f(1/N\cdot\sum^N_{i = 1}\psi(x_i))$, where $f$ and $\psi$ are approximations of two-layer neural networks $\tilde f$ and $\tilde \psi$, respectively. Hence, the action-value function $Q(\mu_p)$ approximately takes the form of deep sets \citep{zaheer2017deep} with $\{x_i\}_{i \in[N]}$ as the set input.

\subsection{Mean-Field Fitted $Q$-Iteration}
In what follows, we establish a value-based algorithm that solves mean-field MARL problem in \S\ref{sec::n_player_sym} based on fitted $Q$-iteration \citep{ernst2005tree}. More specifically, we propose an algorithm that learns the optimal action-value function $Q^*$ by the sample $\{\delta_{a_i}\times p_{i, s}\}_{i\in[n]}$ that follows a sampling distribution $\nu$ over the space of state-action configurations $\tilde\cM(\Omega)$. For each state-action configuration $\delta_{a_i}\times p_{i, s}$, the mean-field state $p_{i, s}\in\cM(\cS)$ is available to us through the observed states $\{s_{i, j}\}_{j\in [N]}$, which are sampled independently from the mean-field state $p_{i, s}$. We further observe the immediate reward $r_i$ and the states $\{s'_{i, j}\}_{j\in [N]}$ that are independently sampled from the mean-field state $p_{i, s'}$.~Here $p_{i, s'}\sim P(\cdot \given a_i, p_{i, s})$ is the mean-field state after transition from the state-action configuration $(a_i, p_{i, s})$. Given the batch of data $\{(\{s_{i, j}\}_{j\in [N]}, a_i, r_i, \{s'_{i, j}\}_{j\in [N]})\}_{i\in[n]}$, the mean-field fitted $Q$-iteration (MF-FQI) sequentially computes
\#\label{eq::FQI_y_i}
\hat y_{i, k} = r_i + \gamma\cdot \max_{a\in\cA} \hat Q^\lambda_k(\mu_{\omega_{a, \hat p_{i, s'}}})
\#
at the $k$-th iteration. Here $\mu_{\omega_{a, \hat p_{i, s'}}}$ is the mean embedding of the distribution $\omega_{a, \hat p_{i, s'}} = \delta_{a}\times \hat p_{i, s'}$, $\hat p_{i, s'}$ is the empirical distribution supported on the set $\{s'_{i, j}\}_{j\in [N]}$, and $Q^\lambda_k$ is the approximation of the optimal action-value function at the $k$-th iteration of MF-FQI. Upon computing $\{\hat y_{i,k}\}_{i\in[n]}$ according to \eqref{eq::FQI_y_i}, MF-FQI then updates the approximation of the optimal action-value function in the RKHS $\cH(K)$ by solving the following optimization problem,
\#\label{eq::FQI_KRR}
 Q^\lambda_{k + 1} = \argmin_{f \in \cH(K)} \frac{1}{n}\sum^n_{i = 1}\bigl(f(\mu_{\hat \omega_i}) - \hat y_{i, k}\bigr)^2 + \lambda\cdot\|f\|^2_{\cH(K)},\quad \hat Q^\lambda_{k+1} = \min\{Q^\lambda_{k + 1}, Q_{\max}\},
\#
where $\hat\omega_i = \delta_a\times \hat p_{i, s}$ and $\hat p_{i, s}$ is the empirical distribution supported on the set $\{s_{i, j}\}_{j\in[N]}$. We summarize MF-FQI defined by \eqref{eq::FQI_y_i} and \eqref{eq::FQI_KRR} in Algorithm \ref{algo::KRR_FQI}. We highlight that MF-FQI has a linear computational complexity in terms of the number of observed agents $N$. Therefore, MF-FQI is computationally tractable even for a large number of observed agents $N$.
\begin{algorithm}[h]
\caption{Mean-Field Fitted $Q$-Iteration (MF-FQI)}\label{algo::KRR_FQI}
\begin{algorithmic}[1]
\STATE{\textbf{Input:}} Batch of data $\{(\{s_{i, j}\}_{j\in [N]}, a_i, r_i, \{s'_{i, j}\}_{j\in [N]})\}_{i\in[n]}$, reproducing kernels $k(\cdot, \cdot)$ and $K(\cdot, \cdot)$, number of iterations $\kappa$, parameter $\lambda$, initial action-value function $\hat Q^\lambda_0$.
\STATE{For all $i\in[n]$, compute mean embeddings $\mu_{\hat\omega_i}$ and $\mu_{a, \hat p_{i, s'}}$ for all $a\in\cA$ as follows,
\$
&\mu_{\hat\omega_i}(\cdot) = \frac{1}{N}\sum^N_{j = 1}k\bigl(\cdot, (a_i, s_{i, j})\bigr), \\
&\mu_{a, \hat p_{i, s'}}(\cdot) = \frac{1}{N}\sum^N_{j = 1}k\bigl(\cdot, (a, s'_{i, j})\bigr), \quad \forall a \in \cA.
\$
}
\FOR{$k = 0, 1, \ldots, \kappa-1$}
\STATE{Compute $\hat y_{i, k} = r_i + \gamma \cdot \max_{a\in\cA} \hat Q^\lambda_k(\mu_{a, \hat p_{i, s'}})$ for all $i\in[n]$.}
\STATE{Update the action-value function $\hat Q^\lambda_{k+1}$ by
\$
 Q^\lambda_{k + 1} = \argmin_{f \in \cH(K)} \frac{1}{n}\sum^n_{i = 1}\bigl(f(\mu_{\hat \omega_i}) - \hat y_{i, k}\bigr)^2 + \lambda\cdot\|f\|^2_{\cH(K)}, \quad \hat Q^\lambda_{k+1} = \min\{Q^\lambda_{k + 1}, Q_{\max}\}.
 \$
 }
\ENDFOR
\STATE{\textbf{Output:}} An estimator $\hat Q^\lambda_\kappa$ of $Q^*$, a greedy policy $\pi_\kappa$ with respect to $\hat Q^\lambda_\kappa$.
\end{algorithmic}
\end{algorithm}

\section{Main Results}\label{sec::theory}
In this section, we establish the theoretical guarantee of MF-FQI defined in Algorithm \ref{algo::KRR_FQI}. In the sequel, we denote by $\hat Q^\lambda_\kappa$ and $\pi_\kappa$ the outputs of MF-FQI. Our goal is to establish an upper bound for $\|Q^* - Q^{\pi_\kappa}\|_{1, \mu}$, where $\mu$ is the measurement distribution over $\tilde\cM(\Omega)$. 

We first introduce the definition of concentration coefficients. For a policy $\pi_1$, we define $E^{\pi_1} \nu$ as the distribution of $\Lambda_1 = \delta_{A_1}\times P_{1, s}$, where $P_{1, s} \sim P(\cdot\given \nu)$ and $A_1\sim\pi_1(\cdot\given P_{1, s})$. Similarly, for policies $\{\pi_i\}_{i\in[\ell]}$, we recursively define $E^{\pi_\ell} \circ E^{\pi_{\ell - 1}}\circ\ldots \circ E^{\pi_1} \nu$ as the distribution of $\Lambda_\ell = \delta_{A_\ell}\times P_{\ell, s}$, where $P_{\ell, s} \sim P(\cdot\given  E^{\pi_{\ell - 1}}\circ\ldots \circ E^{\pi_1} \nu)$ and $A_{\ell} \sim \pi_{\ell}(\cdot\given P_{\ell, s})$. In what follows, we define the concentration coefficients that measures the difference between the sampling distribution $\nu$ and the measurement distribution $\mu$ on $\tilde\cM(\Omega)$.
\begin{assumption}(Concentration Coefficients)
\label{asu::concen}
Let $\nu$ be the sampling distribution on $\tilde \cM(\Omega)$. Let $\mu$ be the measurement distribution on $\tilde\cM(\Omega)$. We assume that for any policies $\{\pi_i\}_{i\in[\ell]}$, the distribution $E^{\pi_\ell} \circ E^{\pi_{\ell - 1}}\circ\ldots \circ E^{\pi_1} \nu$ is absolutely continuous with respect to $\mu$. We define the $\ell$-th concentration coefficients between $\nu$ and $\mu$ as follows,
\#\label{eq::concen_coeff1}
&\phi(\ell; \mu, \nu) = \sup_{\pi_1, \ldots, \pi_\ell}\Biggl(\EE_{\mu}\biggl[\biggl(\frac{\ud E^{\pi_\ell} \circ E^{\pi_{\ell - 1}}\circ \ldots \circ E^{\pi_1} \nu }{\ud \mu}\biggr)^2\biggr]\Biggr)^{1/2}.
\#
We assume that $\phi(\ell; \mu, \nu) < +\infty$ for any $\ell \in \NN$. We further assume that there exist a positive absolute constant $\Phi(\mu, \nu)$ such that
\#\label{eq::concen_coeff1}
\sum^\infty_{\ell = 1}\gamma^{\ell - 1}\cdot\ell\cdot \phi(\ell; \mu, \nu) \leq \Phi(\mu, \nu)/(1 - \gamma)^2.
\#
\end{assumption}
Assumption \ref{eq::concen_coeff1} is a standard assumption in the theoretical analysis of reinforcement learning \citep{szepesvari2005finite, munos2008finite, antos2008learning,  lazaric2010analysis, farahmand2010error, farahmand2016regularized, scherrer2013performance, scherrer2015approximate,yang2019theoretical, chen2019information}. 
Under Assumption \ref{asu::concen}, the following theorem upper bounds the error $\|Q^* - Q^{\pi_{\kappa}}\|_{1, \mu}$ of MF-FQI.
\begin{proposition}[Error Propagation]
\label{thm::error_prop}
Let $\{\hat Q^\lambda_i\}_{i\in [\kappa]}$ be the output of Algorithm \ref{algo::KRR_FQI}. Let $\pi_\kappa$ be the greedy policy corresponding to $\hat Q^\lambda_\kappa$. Under Assumption \ref{asu::concen}, it holds that
\#\label{eq::thm_err_prop}
\|Q^* - Q^{\pi_\kappa}\|_{1, \mu} &\leq \underbrace{\frac{2\gamma\cdot\Phi(\mu, \nu)}{(1 -\gamma)^2}\cdot\max_{i \in [\kappa]}\|\hat Q^\lambda_{i} - T \hat Q^\lambda_{i-1}\|_\nu}_{\textstyle{\rm (a)}}+ \underbrace{\frac{4\gamma^{\kappa+1}\cdot Q_{\max}}{1 - \gamma}}_{\textstyle{\rm (b)}}.
\#
\end{proposition}
\begin{proof}
See \S\ref{pf::err_prop} for a detailed proof.
\end{proof}
Following from Theorem \ref{thm::error_prop}, the error of MF-FQI is upper bounded by the sum of the two terms on the right-hand side of \eqref{eq::thm_err_prop}. Here term (b) characterizes the algorithmic error that hinges on the number of iterations $\kappa$. Meanwhile, term (a) characterizes the one-step approximation error that hinges on the approximation $\hat Q^\lambda_{i}$ of $T \hat Q^\lambda_{i-1}$. In the sequel, we upper bound the one-step approximation error characterized by term (a). To this end, we first impose the following regularity condition on the Bellman optimality operator $T$ and the RKHS $\cH(K)$.
\begin{assumption}[Regularity Condition of $T$ and $\cH(K)$]
\label{asu::kernel_approx}
We define the integral operator $\cC$ as follows,
\#\label{eq::int_oper_C}
\cC f(x) = \int_{\tilde\cM(\Omega)} K(x, \mu_\omega) f(\mu_\omega) d\nu(\omega).
\#
We assume that the eigenvalues $\{t_n\}_{n \in \NN}$ of $\cC$ is bounded such that $\alpha \leq n^{b} t_n \leq \beta$ for all $n\in\NN$, where $\alpha$, $\beta$, and $b > 1$ are positive absolute constants. We further assume that for any output $Q^\lambda\in\cH(K)$ of the regression problem defined in \eqref{eq::FQI_KRR}, it holds for some $g \in \cH(K)$ that
\#\label{eq::reg_QHT}
Q_{\cH, T} = \cC^{(c - 1)/2} g,\quad \|g\|_{\cH(K)} \leq R.
\#
Here $R>0$ and $c\in[1, 2]$ are absolute constants, and $Q_{\cH, T}$ is defined as follows,
\#\label{eq::def_QHT}
Q_{\cH, T} = \Pi_{\cH(K)} (T \hat Q^\lambda) =\argmin_{f\in\cH(K)} \|f - T \hat Q^\lambda\|_{\nu}, \quad \hat Q^{\lambda} = \min\{Q^\lambda, Q_{\max}\},
\#
where we denote by $\Pi_{\cH(K)}$ the projection onto $\cH(K)$ with respect to the norm $\|\cdot\|_{\nu}$.
\end{assumption}
Assumption \ref{asu::kernel_approx} is a mild regularity assumption on the RKHS $\cH(K)$ and the Bellman optimality operator $T$. Similar assumptions arises in the analysis of kernel ridge regression \citep{caponnetto2007optimal, szabo2015two, lin2017distributed}. The parameters $b$ and $c$ in Assumption \ref{asu::kernel_approx} define a prior space $\cP(b, c)$ in the context of kernel ridge regression \citep{caponnetto2007optimal}. Intuitively, the parameter $c$ controls the smoothness of $Q_{\cH, T}$ defined in \eqref{eq::def_QHT}, and the parameter $b$ controls the size of $\cH(K)$. Under Assumption \ref{asu::kernel_approx}, the following theorem characterizes the one-step approximation error of MF-FQI defined in Algorithm \ref{algo::KRR_FQI}.
\begin{theorem}[One-step Approximation Error]
\label{thm::one_step}
Let $\eta$, $\tau$ be two constants such that $ 0<\eta+\tau < 1$. Let $C(\eta) = 32\log^2(6/\eta)$. Under Assumptions \ref{asu::bound_kernel} and \ref{asu::kernel_approx}, for
\#\label{eq::asu_Nnl}
&N \geq 2\varrho\cdot(1 + \sqrt{\log (|\cA|\cdot n/2\tau )})^2\cdot(64L^2\varsigma^2/\lambda^2)^{1/h},\quad n \geq \frac{2C(\eta)\varsigma\beta b}{(b - 1)\lambda^{1 + 1/b}}, \quad \lambda \leq \|\cC\|_{\cH(K)},
\#
it holds with probability at least $1 - \eta - \tau$ that
\#\label{eq::one_step}
\|\hat Q^\lambda_{k} - T \hat Q^\lambda_{k-1}\|^2_{\nu} \leq \cG_1 + \cG_2 + \psi^2_T, \quad \forall k \in[\kappa],
\#
where
\#\label{eq::error_part}
\cG_1 &= \frac{8L^2Q^2_{\max}\bigl(1 + \sqrt{\log (|\cA|\cdot n/2\tau ) }\bigr)^{2h}\cdot (2\varrho)^h}{\lambda \cdot N^h}\cdot\Bigl(1 + \frac{5\varsigma^2}{\lambda^2}\Bigr),\notag\\
\cG_2 &= C(\eta) \cdot \Biggl(R \lambda^c +\frac{\varsigma^2 R }{\lambda^{2 - c}n^2} + \frac{\varsigma R \lambda^{c - 1}}{4n}+ \frac{\varsigma M^2}{\lambda n^2} + \frac{\Sigma^2 \beta b}{(b - 1)n\lambda^{1/b}}\Biggr),
\#
and the term $\psi_T$ is defined as follows,
\#\label{eq::def_distt}
\psi_T = \sup_{k \in [\kappa]}\|T\hat Q^\lambda_k - \Pi_{\cH(K)}(TQ^\lambda_k)\|_{\nu}.
\#
Here $M$ and $\Sigma$ are positive absolute constants, the parameters $\varsigma$, $\varrho$, $h$ are defined in Assumption \ref{asu::bound_kernel}, and the parameters $\cC$, $b$, $c$ are defined in Assumption \ref{asu::kernel_approx}, and $\Pi_{\cH(K)}$ the projection onto $\cH(K)$ with respect to the norm $\|\cdot\|_{\nu}$.
\end{theorem}
\begin{proof}
See \S\ref{pf::one_step} for a detailed proof.
\end{proof}
Following from Theorem \ref{thm::one_step}, the one-step approximation error is upper bounded by the sum of the three terms, $\cG_1$, $\cG_2$, and $\psi_T$, on the right-hand side of \eqref{eq::one_step}. Here the term $\cG_1$ characterizes the error by approximating the mean-field state via $N$ observed agents, the term $\cG_2$ characterizes the error by estimating $TQ$ with the batch of size $n$, and the term $\psi_T$ characterizes the error in of approximating $TQ$ by functions from the RKHS $\cH(K)$. If we further assume that $T Q \in \cH(K)$ for all $Q \in \cH(K)$ with $Q\leq Q_{\max}$, then term $\psi_T$ vanishes and $Q^*\in\cH(K)$ is the unique fixed point of $T$ in $\cH(K)$. Combining the error propagation in Proposition \ref{thm::error_prop} and the one-step approximation error in Theorem \ref{thm::one_step}, we obtain the following theorem that upper bounds the error of MF-FQI defined in Algorithm \ref{algo::KRR_FQI}.
\begin{theorem}[Theoretical Guarantee of MF-FQI]
\label{thm::err_fin}
Let $\pi_\kappa$ be the output of MF-FQI and $n = N^a$. Under the assumptions imposed by Proposition \ref{thm::error_prop} and \ref{thm::one_step}, for
\$
 \|\cC\|_{\cH(K)} \geq \max\bigl\{(a\cdot \log N / N)^{h/(c + 3)}, 1/N^{ab/(bc + 1)}\bigr\}, 
\$
it holds with probability at least $1  - \eta - \tau$ that
\#\label{eq::final_error}
\|Q^* - Q^{\pi_\kappa}\|_{1, \mu} &\leq\frac{2\gamma\cdot\Phi(\mu, \nu)}{(1 -\gamma)^2}\cdot \biggl(\underbrace{C\cdot\Xi}_{\textstyle{\rm (i)}} + \underbrace{\psi_T}_{\textstyle{\rm (ii)}}\biggr)+ \underbrace{\frac{4\gamma^{\kappa+1}\cdot Q_{\max}}{1 - \gamma}}_{\textstyle{\rm (iii)}},
\#
where $C$ is a positive absolute constant, $\psi_T$ is defined in \eqref{eq::def_distt} of Theorem \ref{thm::one_step}, and
\$
\Xi = \max\biggl\{\Bigl(\frac{\log (|\cA|\cdot N/\tau)}{ N}\Bigr)^{\frac{h c}{2(c + 3)}},  1/N^{\frac{abc}{2(bc + 1)}}\biggr\}.
\$
Here the integral operator $\cC$ is defined in \eqref{eq::int_oper_C} of Assumption \ref{asu::kernel_approx}, the parameter $h$ is defined in Assumption \ref{asu::bound_kernel}, and parameters $b$, $c$ are defined in Assumption \ref{asu::kernel_approx}.
\end{theorem}
\begin{proof}
See \S\ref{pf::err_fin} for a detailed proof.
\end{proof}
By Theorem \ref{thm::err_fin}, the approximation error of the action-value function attained by MF-FQI  is characterized by the three terms on the right-hand side of \eqref{eq::final_error}. 
Here term (i) characterizes the statistical error, which is small for a sufficiently large number of observed agents $N$ and batch size $n = N^a$.
Term (ii) is the bias $\psi_T$ defined in in \eqref{eq::def_distt} of Theorem \ref{thm::one_step}, which vanishes if the Bellman optimality operator $T$ is closed in the RKHS $\cH(K)$. 
Term (iii) characterizes the algorithmic error, which is small for a sufficiently large number of iterations $\kappa$. 

In the sequel, we assume that $T$ is closed in $\cH(K)$ and thus $\psi_T = 0$ for simplicity. Note that the algorithmic error characterized by term (ii) has a linear rate of convergence, which is negligible comparing with the statistical error characterized by term (i) if the iteration number is sufficiently large. More specifically, if it holds for some positive absolute constant $C$ that
\#\label{eq::itr_domin}
\kappa \geq C\cdot \max\biggl\{& \frac{h c}{2(c + 3)\cdot  \log (1/\gamma)} \cdot \log\Bigl(\frac{N}{\log(|\cA|\cdot N)}\Bigr),~\frac{abc}{2(bc+1)\cdot  \log (1/\gamma)}\cdot \log N\biggr\},
\#
then the dominating term on the right-hand side of \eqref{eq::final_error} in Theorem \ref{thm::err_fin} is term (i) that characterizes the statistical error.
\vskip4pt
\noindent{\bf Phase Transition. } Note that MF-FQI involves two stage of sampling, where the first stage samples $n$ mean-field state from the sampling distribution $\nu$, and the second stage samples $N$ states from each mean-field state. In what follows, we discuss the connection between the performance of MF-FQI and the sample complexity of the two-stage sampling involved. More specifically, we discuss the phase transition in the statistical error of MF-FQI when $a = \log n/\log N$ transits from zero to infinity. We categorize the phase transition into the following regimes in terms of $a$.
\begin{enumerate}
\item For $a > h\cdot(c + 1/b)/(c +3)$, the rate of convergence of the statistical error takes the form of $\cO((\log(|\cA|\cdot N)/N)^{hc/(2c + 6)})$. In this regime, increasing the number of observations $N$ for each mean-field state improves the performance of MF-FQI, whereas increasing the batch size $n$  of mean-field state cannot improve the performance of MF-FQI. 
\item For $0 < a < h\cdot(c + 1/b)/(c +3)$ , the rate of convergence of the statistical error takes the form of $\cO(1/n^{bc/(2bc + 2)})$. In this regime, increasing batch size $n$ of mean-field state improves the performance of MF-FQI, whereas increasing the number of observations $N$ for each mean-field state cannot improve the performance of MF-FQI.
\end{enumerate}
In conclusion, under regularity conditions, MF-FQI approximately achieves the optimal policy for a sufficiently large number of iteration $\kappa$, batch size $n$, and number of observations $N$. We highlight that MF-FQI enjoys a ``blessing of many agents'' property. More specifically, for a sufficiently large batch size $n$ of the mean-field state, a larger number $N$ of observed agents improves the learning of $Q^*$. 
\section{Limitation and Future Work}
In this work, we propose MF-FQI, which tackles mean-field reinforcement learning problem with symmetric agents and a centralized controller.
Such a setting is extensively studied in the analysis of societal-scale systems \citep{gueant2011mean, gomes2015economic, moll2019uneven}, such as the example of central bank or central government in \S\ref{sec::intro}. MF-FQI tackles the ``curse of many agents" via mean embedding of the mean-field state for the (inexact) policy evaluation step, which approximately calculates the action-value function for the greedy policy. 
Based on the action-value function, we obtain a greedy policy, which corresponds to the policy improvement step. 
Such an approach is intractable when the action space also suffers from the ``curse of many agents", as Q-learning requires taking the maximum over the action space at each iteration, which can be combinatorially large if each agent takes its own action. 
However, the mean embedding technique is still applicable for the policy evaluation step even if each agent takes its own action, which can be coupled with other policy optimization methods, such as policy gradient \cite{sutton2018reinforcement} and proximal policy optimization \cite{schulman2015trust, schulman2017proximal}. 
By replacing the greedy policy improvement step with other policy optimization methods,
we are able to tackle the ``curse of many agents" of both the state space and the action space, which is left as our future research.

\bibliographystyle{ims}
\bibliography{kernel_MFG}

\begin{thebibliography}{81}
\expandafter\ifx\csname natexlab\endcsname\relax\def\natexlab#1{#1}\fi
\expandafter\ifx\csname url\endcsname\relax
  \def\url#1{\texttt{#1}}\fi
\expandafter\ifx\csname urlprefix\endcsname\relax\def\urlprefix{}\fi

\bibitem[{Acciaio et~al.(2018)Acciaio, Backhoff-Veraguas and
  Carmona}]{acciaio2018extended}
\text{Acciaio, B.}, \text{Backhoff-Veraguas, J.} and \text{Carmona, R.} (2018).
\newblock Extended mean field control problems: stochastic maximum principle
  and transport perspective.
\newblock \textit{arXiv preprint arXiv:1802.05754}.

\bibitem[{Altun and Smola(2006)}]{altun2006unifying}
\text{Altun, Y.} and \text{Smola, A.} (2006).
\newblock Unifying divergence minimization and statistical inference via convex
  duality.
\newblock In \textit{International Conference on Computational Learning
  Theory}. Springer.

\bibitem[{Andersson and Djehiche(2011)}]{andersson2011maximum}
\text{Andersson, D.} and \text{Djehiche, B.} (2011).
\newblock A maximum principle for {SDE}s of mean-field type.
\newblock \textit{Applied Mathematics \& Optimization}, \textbf{63} 341--356.

\bibitem[{Antos et~al.(2008)Antos, Szepesv{\'a}ri and
  Munos}]{antos2008learning}
\text{Antos, A.}, \text{Szepesv{\'a}ri, C.} and \text{Munos, R.} (2008).
\newblock Learning near-optimal policies with {B}ellman-residual minimization
  based fitted policy iteration and a single sample path.
\newblock \textit{Machine Learning}, \textbf{71} 89--129.

\bibitem[{Arora et~al.(2019)Arora, Du, Hu, Li and Wang}]{arora2019fine}
\text{Arora, S.}, \text{Du, S.~S.}, \text{Hu, W.}, \text{Li, Z.} and
  \text{Wang, R.} (2019).
\newblock Fine-grained analysis of optimization and generalization for
  overparameterized two-layer neural networks.
\newblock \textit{arXiv preprint arXiv:1901.08584}.

\bibitem[{Bensoussan et~al.(2013)Bensoussan, Frehse, Yam
  et~al.}]{bensoussan2013mean}
\text{Bensoussan, A.}, \text{Frehse, J.}, \text{Yam, P.} \text{et~al.} (2013).
\newblock \textit{Mean field games and mean field type control theory}, vol.
  101.
\newblock Springer.

\bibitem[{Bloem-Reddy and Teh(2019)}]{bloem2019probabilistic}
\text{Bloem-Reddy, B.} and \text{Teh, Y.~W.} (2019).
\newblock Probabilistic symmetry and invariant neural networks.
\newblock \textit{arXiv preprint arXiv:1901.06082}.

\bibitem[{Bu et~al.(2008)Bu, Babu, De~Schutter et~al.}]{bu2008comprehensive}
\text{Bu, L.}, \text{Babu, R.}, \text{De~Schutter, B.} \text{et~al.} (2008).
\newblock A comprehensive survey of multiagent reinforcement learning.
\newblock \textit{IEEE Transactions on Systems, Man, and Cybernetics, Part C
  (Applications and Reviews)}, \textbf{38} 156--172.

\bibitem[{Buckdahn et~al.(2011)Buckdahn, Djehiche and Li}]{buckdahn2011general}
\text{Buckdahn, R.}, \text{Djehiche, B.} and \text{Li, J.} (2011).
\newblock A general stochastic maximum principle for {SDE}s of mean-field type.
\newblock \textit{Applied Mathematics \& Optimization}, \textbf{64} 197--216.

\bibitem[{Buckdahn et~al.(2009)Buckdahn, Djehiche, Li, Peng
  et~al.}]{buckdahn2009mean}
\text{Buckdahn, R.}, \text{Djehiche, B.}, \text{Li, J.}, \text{Peng, S.}
  \text{et~al.} (2009).
\newblock Mean-field backward stochastic differential equations: a limit
  approach.
\newblock \textit{The Annals of Probability}, \textbf{37} 1524--1565.

\bibitem[{Caponnetto and De~Vito(2007)}]{caponnetto2007optimal}
\text{Caponnetto, A.} and \text{De~Vito, E.} (2007).
\newblock Optimal rates for the regularized least-squares algorithm.
\newblock \textit{Foundations of Computational Mathematics}, \textbf{7}
  331--368.

\bibitem[{Carmona and Delarue(2018)}]{carmona2018probabilistic}
\text{Carmona, R.} and \text{Delarue, F.} (2018).
\newblock \textit{Probabilistic Theory of Mean Field Games with Applications
  I-II}.
\newblock Springer.

\bibitem[{Carmona et~al.(2013)Carmona, Delarue and
  Lachapelle}]{carmona2013control}
\text{Carmona, R.}, \text{Delarue, F.} and \text{Lachapelle, A.} (2013).
\newblock Control of {M}c{K}ean--{V}lasov dynamics versus mean field games.
\newblock \textit{Mathematics and Financial Economics}, \textbf{7} 131--166.

\bibitem[{Carmona et~al.(2015)Carmona, Delarue et~al.}]{carmona2015forward}
\text{Carmona, R.}, \text{Delarue, F.} \text{et~al.} (2015).
\newblock Forward--backward stochastic differential equations and controlled
  {M}c{K}ean--{V}lasov dynamics.
\newblock \textit{The Annals of Probability}, \textbf{43} 2647--2700.

\bibitem[{Chen and Jiang(2019)}]{chen2019information}
\text{Chen, J.} and \text{Jiang, N.} (2019).
\newblock Information-theoretic considerations in batch reinforcement learning.
\newblock \textit{arXiv preprint arXiv:1905.00360}.

\bibitem[{Claus and Boutilier(1998)}]{claus1998dynamics}
\text{Claus, C.} and \text{Boutilier, C.} (1998).
\newblock The dynamics of reinforcement learning in cooperative multiagent
  systems.
\newblock \textit{AAAI/IAAI}, \textbf{1998} 746--752.

\bibitem[{D{\v{z}}eroski et~al.(2001)D{\v{z}}eroski, De~Raedt and
  Driessens}]{dvzeroski2001relational}
\text{D{\v{z}}eroski, S.}, \text{De~Raedt, L.} and \text{Driessens, K.} (2001).
\newblock Relational reinforcement learning.
\newblock \textit{Machine Learning}, \textbf{43} 7--52.

\bibitem[{Ernst et~al.(2005)Ernst, Geurts and Wehenkel}]{ernst2005tree}
\text{Ernst, D.}, \text{Geurts, P.} and \text{Wehenkel, L.} (2005).
\newblock Tree-based batch mode reinforcement learning.
\newblock \textit{Journal of Machine Learning Research}, \textbf{6} 503--556.

\bibitem[{Farahmand et~al.(2009)Farahmand, Ghavamzadeh, Szepesv{\'a}ri and
  Mannor}]{massoud2009regularized}
\text{Farahmand, A.-m.}, \text{Ghavamzadeh, M.}, \text{Szepesv{\'a}ri, C.} and
  \text{Mannor, S.} (2009).
\newblock Regularized fitted q-iteration for planning in continuous-space
  markovian decision problems.
\newblock In \textit{2009 American Control Conference}. IEEE.

\bibitem[{Farahmand et~al.(2016)Farahmand, Ghavamzadeh, Szepesv{\'a}ri and
  Mannor}]{farahmand2016regularized}
\text{Farahmand, A.-m.}, \text{Ghavamzadeh, M.}, \text{Szepesv{\'a}ri, C.} and
  \text{Mannor, S.} (2016).
\newblock Regularized policy iteration with nonparametric function spaces.
\newblock \textit{Journal of Machine Learning Research}, \textbf{17}
  4809--4874.

\bibitem[{Farahmand et~al.(2010)Farahmand, Szepesv{\'a}ri and
  Munos}]{farahmand2010error}
\text{Farahmand, A.-m.}, \text{Szepesv{\'a}ri, C.} and \text{Munos, R.} (2010).
\newblock Error propagation for approximate policy and value iteration.
\newblock In \textit{Advances in Neural Information Processing Systems}.

\bibitem[{Fornasier and Solombrino(2014)}]{fornasier2014mean}
\text{Fornasier, M.} and \text{Solombrino, F.} (2014).
\newblock Mean-field optimal control.
\newblock \textit{ESAIM: Control, Optimisation and Calculus of Variations},
  \textbf{20} 1123--1152.

\bibitem[{Fukumizu et~al.(2008)Fukumizu, Gretton, Sun and
  Sch{\"o}lkopf}]{fukumizu2008kernel}
\text{Fukumizu, K.}, \text{Gretton, A.}, \text{Sun, X.} and
  \text{Sch{\"o}lkopf, B.} (2008).
\newblock Kernel measures of conditional dependence.
\newblock In \textit{Advances in Neural Information Processing Systems}.

\bibitem[{G{\"a}rtner et~al.(2002)G{\"a}rtner, Flach, Kowalczyk and
  Smola}]{gartner2002multi}
\text{G{\"a}rtner, T.}, \text{Flach, P.~A.}, \text{Kowalczyk, A.} and
  \text{Smola, A.~J.} (2002).
\newblock Multi-instance kernels.
\newblock In \textit{International Conference on Machine Learning}, vol.~2.

\bibitem[{Gomes et~al.(2015)Gomes, Nurbekyan and Pimentel}]{gomes2015economic}
\text{Gomes, D.~A.}, \text{Nurbekyan, L.} and \text{Pimentel, E.} (2015).
\newblock Economic models and mean-field games theory.
\newblock In \textit{30th Brazilian Mathematics Colloquium}.

\bibitem[{Gretton et~al.(2007)Gretton, Borgwardt, Rasch, Sch{\"o}lkopf and
  Smola}]{gretton2007kernel}
\text{Gretton, A.}, \text{Borgwardt, K.}, \text{Rasch, M.},
  \text{Sch{\"o}lkopf, B.} and \text{Smola, A.~J.} (2007).
\newblock A kernel method for the two-sample-problem.
\newblock In \textit{Advances in Neural Information Processing Systems}.

\bibitem[{Gretton et~al.(2012)Gretton, Borgwardt, Rasch, Sch{\"o}lkopf and
  Smola}]{gretton2012kernel}
\text{Gretton, A.}, \text{Borgwardt, K.~M.}, \text{Rasch, M.~J.},
  \text{Sch{\"o}lkopf, B.} and \text{Smola, A.} (2012).
\newblock A kernel two-sample test.
\newblock \textit{Journal of Machine Learning Research}, \textbf{13} 723--773.

\bibitem[{Gretton et~al.(2009)Gretton, Smola, Huang, Schmittfull, Borgwardt and
  Sch{\"o}lkopf}]{gretton2009covariate}
\text{Gretton, A.}, \text{Smola, A.}, \text{Huang, J.}, \text{Schmittfull, M.},
  \text{Borgwardt, K.} and \text{Sch{\"o}lkopf, B.} (2009).
\newblock Covariate shift by kernel mean matching.
\newblock \textit{Dataset shift in machine learning}, \textbf{3} 5.

\bibitem[{Gu{\'e}ant et~al.(2011)Gu{\'e}ant, Lasry and Lions}]{gueant2011mean}
\text{Gu{\'e}ant, O.}, \text{Lasry, J.-M.} and \text{Lions, P.-L.} (2011).
\newblock Mean field games and applications.
\newblock In \textit{Paris-Princeton lectures on mathematical finance 2010}.
  Springer, 205--266.

\bibitem[{Guestrin et~al.(2002)Guestrin, Lagoudakis and
  Parr}]{guestrin2002coordinated}
\text{Guestrin, C.}, \text{Lagoudakis, M.} and \text{Parr, R.} (2002).
\newblock Coordinated reinforcement learning.
\newblock In \textit{ICML}, vol.~2. Citeseer.

\bibitem[{Guo et~al.(2019)Guo, Hu, Xu and Zhang}]{guo2019learning}
\text{Guo, X.}, \text{Hu, A.}, \text{Xu, R.} and \text{Zhang, J.} (2019).
\newblock Learning mean-field games.
\newblock \textit{arXiv preprint arXiv:1901.09585}.

\bibitem[{Haussler(1999)}]{haussler1999convolution}
\text{Haussler, D.} (1999).
\newblock Convolution kernels on discrete structures.
\newblock Tech. rep., Department of Computer Science, University of California.

\bibitem[{Hu and Wellman(2003)}]{hu2003nash}
\text{Hu, J.} and \text{Wellman, M.~P.} (2003).
\newblock Nash {Q}-learning for general-sum stochastic games.
\newblock \textit{Journal of Machine Learning Research}, \textbf{4} 1039--1069.

\bibitem[{Huang et~al.(2003)Huang, Caines and
  Malham{\'e}}]{huang2003individual}
\text{Huang, M.}, \text{Caines, P.~E.} and \text{Malham{\'e}, R.~P.} (2003).
\newblock Individual and mass behaviour in large population stochastic wireless
  power control problems: centralized and {N}ash equilibrium solutions.
\newblock In \textit{42nd IEEE International Conference on Decision and
  Control}, vol.~1. IEEE.

\bibitem[{Huang et~al.(2007)Huang, Caines and Malham{\'e}}]{huang2007large}
\text{Huang, M.}, \text{Caines, P.~E.} and \text{Malham{\'e}, R.~P.} (2007).
\newblock Large-population cost-coupled {LQG} problems with nonuniform agents:
  individual-mass behavior and decentralized $\varepsilon $-{N}ash equilibria.
\newblock \textit{IEEE Transactions on Automatic Control}, \textbf{52}
  1560--1571.

\bibitem[{Huang et~al.(2012)Huang, Caines and Malham{\'e}}]{huang2012social}
\text{Huang, M.}, \text{Caines, P.~E.} and \text{Malham{\'e}, R.~P.} (2012).
\newblock Social optima in mean field {LQG} control: centralized and
  decentralized strategies.
\newblock \textit{IEEE Transactions on Automatic Control}, \textbf{57}
  1736--1751.

\bibitem[{Jacot et~al.(2018)Jacot, Gabriel and Hongler}]{jacot2018neural}
\text{Jacot, A.}, \text{Gabriel, F.} and \text{Hongler, C.} (2018).
\newblock Neural tangent kernel: Convergence and generalization in neural
  networks.
\newblock In \textit{Advances in neural information processing systems}.

\bibitem[{Jiachen et~al.(2017)Jiachen, Xiaojing, Rakshit, Huan and
  Hongyuan}]{yang2017learning}
\text{Jiachen, Y.}, \text{Xiaojing, Y.}, \text{Rakshit, T.}, \text{Huan, X.}
  and \text{Hongyuan, Z.} (2017).
\newblock Learning deep mean field games for modeling large population
  behavior.
\newblock \textit{arXiv preprint arXiv:1711.03156}.

\bibitem[{Jiang and Lu(2018)}]{jiang2018learning}
\text{Jiang, J.} and \text{Lu, Z.} (2018).
\newblock Learning attentional communication for multi-agent cooperation.
\newblock In \textit{Advances in Neural Information Processing Systems}.

\bibitem[{Kar et~al.(2013)Kar, Moura and Poor}]{kar2013cal}
\text{Kar, S.}, \text{Moura, J.~M.} and \text{Poor, H.~V.} (2013).
\newblock $\cal{Q} \cal{D}$-learning: A collaborative distributed strategy for
  multi-agent reinforcement learning through $\rm{ Consensus}+\rm{
  Innovations}$.
\newblock \textit{IEEE Transactions on Signal Processing}, \textbf{61}
  1848--1862.

\bibitem[{Kondor and Jebara(2003)}]{kondor2003kernel}
\text{Kondor, R.} and \text{Jebara, T.} (2003).
\newblock A kernel between sets of vectors.
\newblock In \textit{International Conference on Machine Learning}.

\bibitem[{Lasry and Lions(2006{\natexlab{a}})}]{lasry2006jeux}
\text{Lasry, J.-M.} and \text{Lions, P.-L.} (2006{\natexlab{a}}).
\newblock Jeux {\`a} champ moyen. i--le cas stationnaire.
\newblock \textit{Comptes Rendus Math{\'e}matique}, \textbf{343} 619--625.

\bibitem[{Lasry and Lions(2006{\natexlab{b}})}]{lasry2006jeux1}
\text{Lasry, J.-M.} and \text{Lions, P.-L.} (2006{\natexlab{b}}).
\newblock Jeux {\`a} champ moyen. ii--horizon fini et contr{\^o}le optimal.
\newblock \textit{Comptes Rendus Math{\'e}matique}, \textbf{343} 679--684.

\bibitem[{Lasry and Lions(2007)}]{lasry2007mean}
\text{Lasry, J.-M.} and \text{Lions, P.-L.} (2007).
\newblock Mean field games.
\newblock \textit{Japanese journal of mathematics}, \textbf{2} 229--260.

\bibitem[{Lauer and Riedmiller(2000)}]{lauer2000algorithm}
\text{Lauer, M.} and \text{Riedmiller, M.} (2000).
\newblock An algorithm for distributed reinforcement learning in cooperative
  multi-agent systems.
\newblock In \textit{International Conference on Machine Learning}. Citeseer.

\bibitem[{Lazaric et~al.(2016)Lazaric, Ghavamzadeh and
  Munos}]{lazaric2010analysis}
\text{Lazaric, A.}, \text{Ghavamzadeh, M.} and \text{Munos, R.} (2016).
\newblock Analysis of classification-based policy iteration algorithms.
\newblock \textit{Journal of Machine Learning Research}, \textbf{17} 583--612.

\bibitem[{Li et~al.(2019)Li, Jiao, Yang, Gong, Wang, Wang, Wu, Ye
  et~al.}]{li2019efficient}
\text{Li, M.}, \text{Jiao, Y.}, \text{Yang, Y.}, \text{Gong, Z.}, \text{Wang,
  J.}, \text{Wang, C.}, \text{Wu, G.}, \text{Ye, J.} \text{et~al.} (2019).
\newblock Efficient ridesharing order dispatching with mean field multi-agent
  reinforcement learning.
\newblock \textit{arXiv preprint arXiv:1901.11454}.

\bibitem[{Lin et~al.(2018)Lin, Zhao, Xu and Zhou}]{lin2018efficient}
\text{Lin, K.}, \text{Zhao, R.}, \text{Xu, Z.} and \text{Zhou, J.} (2018).
\newblock Efficient large-scale fleet management via multi-agent deep
  reinforcement learning.
\newblock In \textit{Proceedings of the 24th ACM SIGKDD International
  Conference on Knowledge Discovery \& Data Mining}. ACM.

\bibitem[{Lin et~al.(2017)Lin, Guo and Zhou}]{lin2017distributed}
\text{Lin, S.-B.}, \text{Guo, X.} and \text{Zhou, D.-X.} (2017).
\newblock Distributed learning with regularized least squares.
\newblock \textit{Journal of Machine Learning Research}, \textbf{18}
  3202--3232.

\bibitem[{Littman(1994)}]{littman1994markov}
\text{Littman, M.~L.} (1994).
\newblock Markov games as a framework for multi-agent reinforcement learning.
\newblock In \textit{Machine Learning Proceedings}. Elsevier, 157--163.

\bibitem[{Meyer-Brandis et~al.(2012)Meyer-Brandis, {\O}ksendal and
  Zhou}]{meyer2012mean}
\text{Meyer-Brandis, T.}, \text{{\O}ksendal, B.} and \text{Zhou, X.~Y.} (2012).
\newblock A mean-field stochastic maximum principle via {M}alliavin calculus.
\newblock \textit{Stochastics An International Journal of Probability and
  Stochastic Processes}, \textbf{84} 643--666.

\bibitem[{Moll et~al.(2019)Moll, Rachel and Restrepo}]{moll2019uneven}
\text{Moll, B.}, \text{Rachel, L.} and \text{Restrepo, P.} (2019).
\newblock Uneven growth: automation’s impact on income and wealth inequality.
\newblock \textit{Manuscript, Princeton University}.

\bibitem[{Muandet et~al.(2012)Muandet, Fukumizu, Dinuzzo and
  Sch{\"o}lkopf}]{muandet2012learning}
\text{Muandet, K.}, \text{Fukumizu, K.}, \text{Dinuzzo, F.} and
  \text{Sch{\"o}lkopf, B.} (2012).
\newblock Learning from distributions via support measure machines.
\newblock In \textit{Advances in Neural Information Processing Systems}.

\bibitem[{Munos and Szepesv{\'a}ri(2008)}]{munos2008finite}
\text{Munos, R.} and \text{Szepesv{\'a}ri, C.} (2008).
\newblock Finite-time bounds for fitted value iteration.
\newblock \textit{Journal of Machine Learning Research}, \textbf{9} 815--857.

\bibitem[{Nash(1951)}]{nash1951non}
\text{Nash, J.} (1951).
\newblock Non-cooperative games.
\newblock \textit{Annals of Mathematics} 286--295.

\bibitem[{Neyshabur et~al.(2018)Neyshabur, Li, Bhojanapalli, LeCun and
  Srebro}]{neyshabur2018towards}
\text{Neyshabur, B.}, \text{Li, Z.}, \text{Bhojanapalli, S.}, \text{LeCun, Y.}
  and \text{Srebro, N.} (2018).
\newblock Towards understanding the role of over-parametrization in
  generalization of neural networks.
\newblock \textit{arXiv preprint arXiv:1805.12076}.

\bibitem[{Ng(1975)}]{ng1975bentham}
\text{Ng, Y.-K.} (1975).
\newblock {B}entham or {B}ergson? {F}inite sensibility, utility functions and
  social welfare functions.
\newblock \textit{The Review of Economic Studies}, \textbf{42} 545--569.

\bibitem[{OpenAI(2018)}]{OpenAI_dota}
\text{OpenAI} (2018).
\newblock Openai {F}ive.
\newblock \url{https://blog.openai.com/openai-five/}.

\bibitem[{Panait and Luke(2005)}]{panait2005cooperative}
\text{Panait, L.} and \text{Luke, S.} (2005).
\newblock Cooperative multi-agent learning: The state of the art.
\newblock \textit{Autonomous Agents and Multi-Agent Systems}, \textbf{11}
  387--434.

\bibitem[{Scherrer(2013)}]{scherrer2013performance}
\text{Scherrer, B.} (2013).
\newblock On the performance bounds of some policy search dynamic programming
  algorithms.
\newblock \textit{arXiv preprint arXiv:1306.0539}.

\bibitem[{Scherrer et~al.(2015)Scherrer, Ghavamzadeh, Gabillon, Lesner and
  Geist}]{scherrer2015approximate}
\text{Scherrer, B.}, \text{Ghavamzadeh, M.}, \text{Gabillon, V.}, \text{Lesner,
  B.} and \text{Geist, M.} (2015).
\newblock Approximate modified policy iteration and its application to the game
  of {T}etris.
\newblock \textit{Journal of Machine Learning Research}, \textbf{16}
  1629--1676.

\bibitem[{Schulman et~al.(2015)Schulman, Levine, Abbeel, Jordan and
  Moritz}]{schulman2015trust}
\text{Schulman, J.}, \text{Levine, S.}, \text{Abbeel, P.}, \text{Jordan, M.}
  and \text{Moritz, P.} (2015).
\newblock Trust region policy optimization.
\newblock In \textit{International Conference on Machine Learning}.

\bibitem[{Schulman et~al.(2017)Schulman, Wolski, Dhariwal, Radford and
  Klimov}]{schulman2017proximal}
\text{Schulman, J.}, \text{Wolski, F.}, \text{Dhariwal, P.}, \text{Radford, A.}
  and \text{Klimov, O.} (2017).
\newblock Proximal policy optimization algorithms.
\newblock \textit{arXiv preprint arXiv:1707.06347}.

\bibitem[{Shalev-Shwartz et~al.(2016)Shalev-Shwartz, Shammah and
  Shashua}]{shalev2016safe}
\text{Shalev-Shwartz, S.}, \text{Shammah, S.} and \text{Shashua, A.} (2016).
\newblock Safe, multi-agent, reinforcement learning for autonomous driving.
\newblock \textit{arXiv preprint arXiv:1610.03295}.

\bibitem[{Silver et~al.(2016)Silver, Huang, Maddison, Guez, Sifre, Van
  Den~Driessche, Schrittwieser, Antonoglou, Panneershelvam, Lanctot
  et~al.}]{silver2016mastering}
\text{Silver, D.}, \text{Huang, A.}, \text{Maddison, C.~J.}, \text{Guez, A.},
  \text{Sifre, L.}, \text{Van Den~Driessche, G.}, \text{Schrittwieser, J.},
  \text{Antonoglou, I.}, \text{Panneershelvam, V.}, \text{Lanctot, M.}
  \text{et~al.} (2016).
\newblock Mastering the game of {G}o with deep neural networks and tree search.
\newblock \textit{nature}, \textbf{529} 484.

\bibitem[{Silver et~al.(2017)Silver, Schrittwieser, Simonyan, Antonoglou,
  Huang, Guez, Hubert, Baker, Lai, Bolton et~al.}]{silver2017mastering}
\text{Silver, D.}, \text{Schrittwieser, J.}, \text{Simonyan, K.},
  \text{Antonoglou, I.}, \text{Huang, A.}, \text{Guez, A.}, \text{Hubert, T.},
  \text{Baker, L.}, \text{Lai, M.}, \text{Bolton, A.} \text{et~al.} (2017).
\newblock Mastering the game of {G}o without human knowledge.
\newblock \textit{Nature}, \textbf{550} 354.

\bibitem[{Smola et~al.(2007)Smola, Gretton, Song and
  Sch{\"o}lkopf}]{smola2007hilbert}
\text{Smola, A.}, \text{Gretton, A.}, \text{Song, L.} and \text{Sch{\"o}lkopf,
  B.} (2007).
\newblock A {H}ilbert space embedding for distributions.
\newblock In \textit{International Conference on Algorithmic Learning Theory}.
  Springer.

\bibitem[{Sriperumbudur et~al.(2010)Sriperumbudur, Gretton, Fukumizu,
  Sch{\"o}lkopf and Lanckriet}]{sriperumbudur2010hilbert}
\text{Sriperumbudur, B.~K.}, \text{Gretton, A.}, \text{Fukumizu, K.},
  \text{Sch{\"o}lkopf, B.} and \text{Lanckriet, G.~R.} (2010).
\newblock Hilbert space embeddings and metrics on probability measures.
\newblock \textit{Journal of Machine Learning Research}, \textbf{11}
  1517--1561.

\bibitem[{Sutton and Barto(2018)}]{sutton2018reinforcement}
\text{Sutton, R.~S.} and \text{Barto, A.~G.} (2018).
\newblock \textit{Reinforcement Learning: An Introduction}.
\newblock MIT press.

\bibitem[{Szab{\'o} et~al.(2015)Szab{\'o}, Gretton, P{\'o}czos and
  Sriperumbudur}]{szabo2015two}
\text{Szab{\'o}, Z.}, \text{Gretton, A.}, \text{P{\'o}czos, B.} and
  \text{Sriperumbudur, B.} (2015).
\newblock Two-stage sampled learning theory on distributions.
\newblock In \textit{Artificial Intelligence and Statistics}.

\bibitem[{Szepesv{\'a}ri and Munos(2005)}]{szepesvari2005finite}
\text{Szepesv{\'a}ri, C.} and \text{Munos, R.} (2005).
\newblock Finite time bounds for sampling based fitted value iteration.
\newblock In \textit{International Conference on Machine Learning}.

\bibitem[{Tan(1993)}]{tan1993multi}
\text{Tan, M.} (1993).
\newblock Multi-agent reinforcement learning: Independent vs. cooperative
  agents.
\newblock In \textit{International Conference on Machine Learning}.

\bibitem[{Tolstikhin et~al.(2017)Tolstikhin, Sriperumbudur and
  Muandet}]{tolstikhin2017minimax}
\text{Tolstikhin, I.}, \text{Sriperumbudur, B.~K.} and \text{Muandet, K.}
  (2017).
\newblock Minimax estimation of kernel mean embeddings.
\newblock \textit{Journal of Machine Learning Research}, \textbf{18}
  3002--3048.

\bibitem[{Vinyals et~al.(2019)Vinyals, Babuschkin, Chung, Mathieu, Jaderberg,
  Czarnecki, Dudzik, Huang, Georgiev, Powell, Ewalds, Horgan, Kroiss,
  Danihelka, Agapiou, Oh, Dalibard, Choi, Sifre, Sulsky, Vezhnevets, Molloy,
  Cai, Budden, Paine, Gulcehre, Wang, Pfaff, Pohlen, Wu, Yogatama, Cohen,
  McKinney, Smith, Schaul, Lillicrap, Apps, Kavukcuoglu, Hassabis and
  Silver}]{alphastarblog}
\text{Vinyals, O.}, \text{Babuschkin, I.}, \text{Chung, J.}, \text{Mathieu,
  M.}, \text{Jaderberg, M.}, \text{Czarnecki, W.~M.}, \text{Dudzik, A.},
  \text{Huang, A.}, \text{Georgiev, P.}, \text{Powell, R.}, \text{Ewalds, T.},
  \text{Horgan, D.}, \text{Kroiss, M.}, \text{Danihelka, I.}, \text{Agapiou,
  J.}, \text{Oh, J.}, \text{Dalibard, V.}, \text{Choi, D.}, \text{Sifre, L.},
  \text{Sulsky, Y.}, \text{Vezhnevets, S.}, \text{Molloy, J.}, \text{Cai, T.},
  \text{Budden, D.}, \text{Paine, T.}, \text{Gulcehre, C.}, \text{Wang, Z.},
  \text{Pfaff, T.}, \text{Pohlen, T.}, \text{Wu, Y.}, \text{Yogatama, D.},
  \text{Cohen, J.}, \text{McKinney, K.}, \text{Smith, O.}, \text{Schaul, T.},
  \text{Lillicrap, T.}, \text{Apps, C.}, \text{Kavukcuoglu, K.},
  \text{Hassabis, D.} and \text{Silver, D.} (2019).
\newblock {AlphaStar: Mastering the Real-Time Strategy Game StarCraft II}.
\newblock
  \url{https://deepmind.com/blog/alphastar-mastering-real-time-strategy-game-starcraft-ii/}.

\bibitem[{Wang and Sandholm(2003)}]{wang2003reinforcement}
\text{Wang, X.} and \text{Sandholm, T.} (2003).
\newblock Reinforcement learning to play an optimal {N}ash equilibrium in team
  {M}arkov games.
\newblock In \textit{Advances in Neural Information Processing Systems}.

\bibitem[{Yang and Gu(2004)}]{yang2004multiagent}
\text{Yang, E.} and \text{Gu, D.} (2004).
\newblock Multiagent reinforcement learning for multi-robot systems: A survey.
\newblock Tech. rep., University of Essex.

\bibitem[{Yang et~al.(2018)Yang, Luo, Li, Zhou, Zhang and Wang}]{yang2018mean}
\text{Yang, Y.}, \text{Luo, R.}, \text{Li, M.}, \text{Zhou, M.}, \text{Zhang,
  W.} and \text{Wang, J.} (2018).
\newblock Mean field multi-agent reinforcement learning.
\newblock \textit{arXiv preprint arXiv:1802.05438}.

\bibitem[{Yang et~al.(2019)Yang, Xie and Wang}]{yang2019theoretical}
\text{Yang, Z.}, \text{Xie, Y.} and \text{Wang, Z.} (2019).
\newblock A theoretical analysis of deep {Q}-learning.
\newblock \textit{arXiv preprint arXiv:1901.00137}.

\bibitem[{Zaheer et~al.(2017)Zaheer, Kottur, Ravanbakhsh, Poczos, Salakhutdinov
  and Smola}]{zaheer2017deep}
\text{Zaheer, M.}, \text{Kottur, S.}, \text{Ravanbakhsh, S.}, \text{Poczos,
  B.}, \text{Salakhutdinov, R.~R.} and \text{Smola, A.~J.} (2017).
\newblock Deep sets.
\newblock In \textit{Advances in Neural Information Processing Systems}.

\bibitem[{Zambaldi et~al.(2018)Zambaldi, Raposo, Santoro, Bapst, Li,
  Babuschkin, Tuyls, Reichert, Lillicrap, Lockhart
  et~al.}]{zambaldi2018relational}
\text{Zambaldi, V.}, \text{Raposo, D.}, \text{Santoro, A.}, \text{Bapst, V.},
  \text{Li, Y.}, \text{Babuschkin, I.}, \text{Tuyls, K.}, \text{Reichert, D.},
  \text{Lillicrap, T.}, \text{Lockhart, E.} \text{et~al.} (2018).
\newblock Relational deep reinforcement learning.
\newblock \textit{arXiv preprint arXiv:1806.01830}.

\bibitem[{Zhang et~al.(2016)Zhang, Bengio, Hardt, Recht and
  Vinyals}]{zhang2016understanding}
\text{Zhang, C.}, \text{Bengio, S.}, \text{Hardt, M.}, \text{Recht, B.} and
  \text{Vinyals, O.} (2016).
\newblock Understanding deep learning requires rethinking generalization.
\newblock \textit{arXiv preprint arXiv:1611.03530}.

\end{thebibliography}
\newpage
\onecolumn
\appendix
\section{Topological Structures}
\label{sec::topo_str}
We now establish a topological structure on the space $\tilde\cM(\Omega)$ adopted from \cite{szabo2015two}. Recall that we denote by $\Omega = \cS \times \cA$ the space of state-action pairs. We assume that $\Omega$ is a polish space, and denote by $\cB(\Omega)$ the Borel $\sigma$-algebra of $\Omega$. We denote by $\cM_{0}(\cA)$ the space of all the point mass distributions on $\cA$, and denote by $\cM(\cS)$ the space of all the distributions on $\cS$. We assume that both $\cM_{0}(\cA)$ and $\cM(\cS)$ are equipped with the weak topology such that for all $f \in \cC_B(\cA)$ and $g \in \cC_B(\cS)$, the mappings $p\mapsto\int f(x) \ud p(x)$ and $q\mapsto\int g(x) \ud q(x)$ are continuous for $p\in\cM_{0}(\Omega)$ and $q\in\cM(\Omega)$, respectively. Note that any $p \in \cM_0(\cA)$ and $q\in\cM(\cS)$ defines a product measure $\omega = p\times q \in \tilde \cM(\Omega)$ on $(\Omega, \cB(\cA)\otimes\cB(\cS))$. We endow the set $\tilde\cM(\Omega)$ with the product topology of corresponding weak topology defined on $\cM_0(\cA)$ and $\cM(\cS)$, which makes $\tilde\cM(\Omega)$ a Polish space.

\section{Proof of Main Results}
\label{sec::pf}
\subsection{Proof of Proposition \ref{prop::inv_pi}}
\label{sec::pf_inv_pi}
We first show that if \eqref{eq::n_player_mdp_invar1} holds, then $Q^*(s, a) = Q^*(\sigma(s), a)$ for any $s \in \cS$ and $\sigma \in \SSS_n$, where $\SSS_n$ is the permutation group of order $n$. For any $\sigma \in \SSS_n$ ,we define the function $f(s, a) = Q^*(\sigma^{-1}(s), a)$. Therefore, it holds that
\$
f\bigl(\sigma(s), a\bigr) = Q^*\bigl(\sigma^{-1}\circ\sigma(s), a\bigr) = Q^*(s, a).
\$
Then, following from the Bellman optimality equation in \eqref{eq::def_bellman_optimal}, it holds that
\#\label{eq::inv_f_nplayer}
f\bigl(\sigma(s), a\bigr) = \EE\Bigl[r(s, a) + \gamma\cdot\max_{a\in\cA} f\bigl(\sigma(S'), a\bigr)\Bigr],
\#
where $S' \sim P(\cdot|s, a)$. We denote by $S'' \sim P(\cdot|\sigma(s), a)$. Following from \eqref{eq::def_bellman_optimal}, it holds that
\#\label{eq::inv_s1}
\PP(s'~|~s, a) = \PP\bigl(\sigma(s')~\bigl|~\sigma(s), a\bigr).
\#
Following from \eqref{eq::inv_s1}, it then holds that $S'' \stackrel{\ud}{=} \sigma(S')$. Meanwhile, it holds from \eqref{eq::def_bellman_optimal} that $r(s, a) \stackrel{\ud}{=} r(\sigma(s), a)$. Therefore, following from \eqref{eq::inv_f_nplayer}, we obtain that
\#\label{eq::inv_f_nplayer1}
f\bigl(\sigma(s), a\bigr) = \EE\Bigl[r\bigl(\sigma(s), a\bigr) + \gamma\cdot\max_{a\in\cA} f\bigl(S'', a\bigr)\Bigr],
\#
where $S'' \sim P(\cdot|\sigma(s), a)$. Note that \eqref{eq::inv_f_nplayer1} holds for any $\sigma \in \SSS_n$ and $s \in \cS$. Therefore, the function $f(s, a)$ is an optimal action-value function. Following from the uniqueness of the optimal action-value function $Q^*$ for a discounted MDP, it holds that
\#\label{eq::inv_f_nplayer2}
f(s, a) = Q^*(\sigma^{-1}(s), a) = Q^*(s, a).
\#
Following from \eqref{eq::inv_f_nplayer2}, it holds for all $\sigma \in \SSS_n$ that
\#
\pi^*(a~|~s) = \argmax_{a\in\cA} Q^*(s, a) = \argmax_{a\in\cA} Q^*\bigl(\sigma(s), a\bigr)  = \pi^*\bigl(a~|~\sigma(s)\bigr),
\#
which conclude the proof of Proposition \ref{prop::inv_pi}.

\subsection{Proof of Proposition \ref{thm::inv::rep}}
\label{sec::pf_inv_rep}
Note that the policy $\pi(a|s)$ together with $\PP(\cdot|S =s, A =a)$ defines a Markov process on $(s, a)$ with transition dynamics
\$
S_{t+1} \sim P(\cdot~|~S_t, A_t), \quad A_{t+1}\sim \pi(\cdot~|~S_{t+1}).
\$ 
We denote by $R(S_0, A_0)$ the discounted total reward following the MDP with initial state $S_0$ and $A_0$, which takes the form
\#\label{eq::prof_inv_R}
R(S_0, A_0) = \sum^\infty_{t = 0} \gamma^t\cdot r(S_t, A_t).
\#
Moreover, following from the definition of action-value function in \eqref{eq::def_Q_nplayer}, it holds that
\#\label{eq::prof_inv_Q}
Q(s, a) = \EE\bigl[R(S, A)~\bigl|~S = s, A = a\bigr].
\#
Note that for any invariant policy $\pi(a\given s) \stackrel{\ud}{=}\pi(a\given \sigma(s))$, $R$ is invariant to the permutation of state $S_0$. With a slight abuse of notation, we denote by $R(S, a) = R(S, A) \given  A = a$. Then following from the definition of $R(S, A)$ in \eqref{eq::prof_inv_R}, we obtain that
\#\label{eq::pf_inv_cond}
R(S_0, a)  \stackrel{d}{=} R\bigl(\sigma(S_0), a\bigr) , \quad \sigma \in \SSS_n,
\#
which holds for any $a\in\cA$. Following from Theorem 12 in \cite{bloem2019probabilistic}, it holds for some $f_a$ that
\#
R(S_0, a) = f_a\bigl(\eta, M(S_0)\bigr), \quad \eta \sim \text{Unif}[0, 1].
\#
Here $M_{S_0}$ is the maximal invariant on $S$ under the permutation group $S_N$, which takes the form of \citep{bloem2019probabilistic}
\$
M_{S_0}(\cdot) = \frac{1}{N}\sum^N_{i = 1}\delta_{S_{0, i}},
\$ 
where $S_{0, i}$ is the $i$-th component of $S_0$, and $\delta_x$ is the point mass at $x$. Following from \eqref{eq::prof_inv_Q}, we complete the proof of Proposition \ref{thm::inv::rep} by setting $g(M, a) = \EE_{\eta}[f_a(\eta, M)\given S_0 = s]$.

\subsection{Proof of Proposition \ref{thm::error_prop}}
\label{pf::err_prop}
The proof strategy is similar to that of \cite{massoud2009regularized, farahmand2010error, farahmand2016regularized, yang2019theoretical}. We denote by $\hat Q^\lambda_{\kappa}$ the output of MF-FQI in the $\kappa$-th iteration. We further denote by
\$
\epsilon_\kappa = T \hat Q^\lambda_{\kappa-1} - \hat Q^\lambda_\kappa.
\$ 
The function $\epsilon_\kappa$ is then the one-step error that corresponds to the $\kappa$-th iteration. Recall that we denote by $\pi_\kappa$ the greedy policy with respect to $\hat Q^\lambda_{\kappa}$, and $Q^{\pi_\kappa}$ the action-value function with respect to the policy $\pi_\kappa$. For each function $Q$ and policy $\pi$, we define the bellman operator $T^\pi$ and evaluation operator $E^\pi$ as follows,
\#\label{eq::def_bellman_eval}
&E^\pi Q(\delta_a\times p_s) = \EE\bigl[Q(\delta_{A'}\times P_{s'}) ~\bigl|~ A'\sim \pi(\cdot~|~p_s), ~P_{s'}\sim P(\cdot~|~\delta_a\times p_s) \bigr],\notag\\
&T^\pi Q(\delta_a \times p_s) = \EE\bigl[r(\delta_a\times p_s)\bigr] + \gamma\cdot E^\pi Q(\delta_a\times p_s).
\#
Recall that we denote by $Q^*$ and $\pi^*$ the optimal action-value function and corresponding greedy policy. In what follows, we establish the error bound $\|Q^* - Q^{\pi_\kappa}\|_{1, \mu}$, where $\mu$ is a distribution over the space $\tilde M(\Omega)$. The idea is to first establish a recursive relation between $Q^* - \hat Q^\lambda_{\kappa}$, and then upper bound the error of $|Q^* - Q^{\pi_\kappa}|$ by  $|Q^* - \hat Q^\lambda_{\kappa}|$. We denote by $\pi_Q$ the greedy policy with respect to $Q$. It then holds that
\#\label{eq::tpi_eqn_t}
T^{\pi_Q}(\delta_a \times p_s) &= \EE\bigl[r(\delta_a\times p_s)\bigr] + \gamma\cdot \EE\Bigl[\max_{a'\in \cA}Q(\delta_{a'}\times P_{s'}) ~\bigl|~P_{s'} \sim P(\cdot~|~\delta_a\times p_s) \Bigr] \notag\\
&= TQ(\delta_a\times p_s).
\#
Meanwhile, it holds for any policy $\pi$ that
\#\label{eq::tpi_grt_pi}
\EE\Bigl[\max_{a'\in \cA}Q(\delta_{a'}\times P_{s'}) ~\Bigl|~P_{s'} \sim P(\cdot~|~\delta_a\times p_s) \Bigr] \geq E^\pi Q(\delta_a\times p_s).
\#
Combining \eqref{eq::tpi_eqn_t} and \eqref{eq::tpi_grt_pi}, we obtain the following identity,
\#\label{eq::pf_errprop_id1}
T^{\pi_Q}Q(\omega) = TQ(\omega) \geq T^{\pi}Q(\omega),
\#
which holds for any $\omega \in \tilde\cM(\Omega)$ and policy $\pi$. Note that $T^{\pi^*}Q^* = Q^*$. Therefore, following from \eqref{eq::pf_errprop_id1}, we obtain that
\#\label{eq::propa_upper_qk}
Q^* - \hat Q^\lambda_{\kappa+1} = Q^* - T\hat Q^\lambda_k + \epsilon_{\kappa+1} &\leq T^{\pi^*}Q^* - T^{\pi^*} \hat Q^\lambda_\kappa + \epsilon_{\kappa +1} \notag\\
&= \gamma \cdot E^{\pi^*}(Q^* - \hat Q^\lambda_\kappa) + \epsilon_{\kappa+1},
\#
where the last equality follows from the definition of Bellman operator in \eqref{eq::def_bellman_eval}. Similarly, it holds that
\#\label{eq::propa_lower_qk}
Q^* - \hat Q^\lambda_{\kappa+1} = T^{\pi^*}Q^* - T\hat Q^\lambda_\kappa + \epsilon_{\kappa+1} &\geq T^{\pi_k}Q^* - T^{\pi_\kappa}\hat Q^\lambda_\kappa + \epsilon_{\kappa+ 1} \notag\\
&= \gamma \cdot E^{\pi_\kappa}(Q^* - \hat Q^\lambda_\kappa) + \epsilon_{\kappa+1},
\#
where recall that we denote by $\pi_k$ the greedy policy with respect to $\hat Q^\lambda_\kappa$. Following the definition of $E^\pi$ in \eqref{eq::def_bellman_eval}, it holds that the operator $E^\pi$ is linear. Therefore, upon iteratively applying \eqref{eq::propa_upper_qk} and \eqref{eq::propa_lower_qk}, it holds for any $K > \kappa$ that
\#\label{eq::err_qk}
&Q^* - \hat Q^\lambda_{K} \leq \gamma^{K - \kappa}\cdot (E^{\pi^*})^{K - \kappa}(Q^* - \hat Q^\lambda_\kappa) + \sum^{K-1}_{i = \kappa}\gamma^{K -1-i}\cdot(E^{\pi^*})^{K - 1-i}\epsilon_{i + 1},\notag\\
&Q^* - \hat Q^\lambda_{K} \geq \gamma^{K - \kappa}\cdot \biggl(\prod^{K-1}_{i = \kappa}E^{\pi_i}\biggr)(Q^* - \hat Q^\lambda_\kappa) + \sum^{K-1}_{i = \kappa}\gamma^{K -1-i}\cdot\biggl(\prod^{K-1}_{j = i+1}E^{\pi_j}\biggr)\epsilon_{i + 1},
\#
where we denote by $\prod^{K-1}_{j = i+1}E^{\pi_j}$ the composition of operators $E^{\pi_{K-1}}\circ E^{\pi_{K-2}}\circ\ldots \circ E^{\pi_{i+1}}$. We now bound the error $Q^* - Q^{\pi_\kappa}$, where $\pi_\kappa$ is the greedy policy with respect to $\hat Q^\lambda_\kappa$. Following from \eqref{eq::pf_errprop_id1} and the linearity of Bellman operators, it holds that
\#\label{eq::err_qpik_eq1}
Q^* - Q^{\pi_\kappa} &= T^{\pi^*}Q^* - T^{\pi_\kappa}Q^{\pi_\kappa} = T^{\pi^*}(Q^* - \hat Q^\lambda_\kappa) + (T^{\pi^*} - T^{\pi_\kappa})\hat Q^\lambda_\kappa + T^{\pi_\kappa}(\hat Q^\lambda_\kappa - Q^{\pi_\kappa})\notag\\
&\leq T^{\pi^*}(Q^* - \hat Q^\lambda_\kappa) + T^{\pi_\kappa}(\hat Q^\lambda_\kappa - Q^{\pi_\kappa}).
\#
Meanwhile, it follows from the definition of $Q^*$ that $Q^* \geq Q^{\pi_\kappa}$. Combining with \eqref{eq::err_qpik_eq1}, we obtain that
\#\label{eq::err_qpik_eq2}
0 \leq Q^* - Q^{\pi_\kappa} &\leq  T^{\pi^*}(Q^* - \hat Q^\lambda_\kappa) + T^{\pi_\kappa}(\hat Q^\lambda_\kappa - Q^{\pi_\kappa})\notag\\
&= \gamma\cdot E^{\pi^*}(Q^* - \hat Q^\lambda_\kappa) + \gamma\cdot E^{\pi_\kappa}(\hat Q^\lambda_\kappa - Q^{\pi_\kappa})\notag\\
&= \gamma \cdot(E^{\pi^*} - E^{\pi_\kappa})(Q^* - \hat Q^\lambda_\kappa) + \gamma\cdot E^{\pi_\kappa}( Q^* - Q^{\pi_\kappa}).
\#
where the first equality follows from the definition of Bellman operator in \eqref{eq::def_bellman_eval}. Note that the Bellman operator $T^\pi$ is contractive with respect to any given policy $\pi$ and bounded reward. It thus holds that the operator $I - \gamma E^{\pi_k}$ is invertible, where $I$ is the identity mapping. Following from \eqref{eq::err_qpik_eq2}, we obtain that
\#\label{eq::err_qpik_eq3}
Q^* - Q^{\pi_\kappa} \leq \gamma\cdot(I - \gamma E^{\pi_\kappa})^{-1} (E^{\pi^*} - E^{\pi_\kappa})(Q^* - \hat Q^\lambda_\kappa),
\#
which holds for any $\kappa> 0$. Following from the definition in \eqref{eq::def_bellman_eval}, it holds for any $\pi$ and $Q_1 \geq Q_2$ that $E^\pi Q_1 \geq E^\pi Q_2$. Then following from the infinite expansion of $(I - \gamma E^{\pi_k})^{-1}$, it holds for any $\pi$ and $Q_1\geq Q_2$ that 
\$
(I - \gamma E^{\pi_\kappa})^{-1}Q_1 \geq (I - \gamma E^{\pi_\kappa})^{-1} Q_2.
\$
Therefore, combining \eqref{eq::err_qk} and \eqref{eq::err_qpik_eq3}, we obtain for $\kappa >0$ that
\#\label{eq::err_qpik}
0 \leq Q^* - Q^{\pi_\kappa} &\leq (I - \gamma E^{\pi_\kappa})^{-1}\biggl\{\sum^{\kappa-1}_{i = 0} \gamma^{\kappa-i}\cdot\biggl((E^{\pi^*})^{\kappa -i} - \prod^{\kappa}_{j = i+1}E^{\pi_j}\biggr)\epsilon_{i + 1}\notag\\
&\quad + \gamma^{\kappa+1}\cdot \biggl((E^{\pi^*})^{\kappa+1} - \prod^{\kappa}_{j = 0}E^{\pi_j}\biggr)(Q^* - \hat Q^\lambda_0)\biggr\}.
\#
For notational simplicity, we introduce the following shorthands,
\#\label{eq::pf_def_operator_JF}
&J_\kappa = (I - \gamma E^{\pi_\kappa})^{-1}\biggl((E^{\pi^*})^{\kappa+1} - \prod^{\kappa}_{j = 0}E^{\pi_j}\biggr), \notag\\
&F_{i, \kappa} =  (I - \gamma E^{\pi_k})^{-1}\biggl((E^{\pi^*})^{\kappa -i} - \prod^{\kappa}_{j = i+1}E^{\pi_j}\biggr), \quad i = 0, \ldots, \kappa-1.
\#
Based on the shorthands defined in \eqref{eq::pf_def_operator_JF}, we rewrite \eqref{eq::err_qpik} as follows,
\#\label{eq::err_qpik_1}
0 \leq Q^* - Q^{\pi_\kappa} &\leq\gamma^{\kappa+1}\cdot  J_\kappa(Q^* - \hat Q^\lambda_0) +   \sum^{\kappa-1}_{i = 0} \gamma^{\kappa-i}F_{i, \kappa}(\epsilon_{i + 1}).
\#
In what follows, we bound the error $\|Q^* - Q^{\pi_\kappa}\|_{1, \mu}$ based on \eqref{eq::err_qpik}, where $\mu$ is a given distribution on $\tilde \cM(\Omega)$. Recall that for a function $f$ defined on the support of a distribution $\mu$, we define the shorthand $\mu(f) = \EE_{\omega\sim\mu}[f(\omega)]$. Following from \eqref{eq::err_qpik_1} and the linearity of operators $F_{i, \kappa}$, $J_\kappa$, and $\mu(\cdot)$, we obtain that
\#\label{eq::err_qpik_2}
\|Q^* - Q^{\pi_\kappa}\|_{1, \mu} = \mu(|Q^* - Q^{\pi_\kappa}|) \leq \gamma^{\kappa+1}\cdot \mu( J_\kappa|Q^* - \hat Q^\lambda_0|) +  \sum^{\kappa-1}_{i = 0} \gamma^{\kappa-i} \mu(F_{i, \kappa}|\epsilon_{i + 1}|).
\#
Note that the operator $E^\pi$ defined in \eqref{eq::def_bellman_eval} is a Markov transition operator defined on $\tilde \cM(\Omega)$ with transition dynamics given by
\$
P_{s, t+1} \sim P(\cdot~|~A_t, P_{s, t}), \quad A_{t+1}\sim \pi(\cdot~|~P_{s, t+1}),
\$
Thus, the operator $E^\pi$ is a Markov transition kernel defined on $\tilde\cM(\Omega)$. The composition $\mu(E^\pi) = E^\pi\cdot \mu$ defines a probability measure, which is a transition over the initial probability distribution $\mu$ on $\tilde \cM(\Omega)$. Indeed, for any $f(\cdot)$ defined on $\tilde \cM(\Omega)$, we denote by $X_0 \sim \mu$ the initial distribution of state-action configurations. It then holds that
\#
\mu(E^\pi f) = \EE\bigl[f(X_1)\bigr] = \mu_1(f),
\# 
where $X_1$ is the transition of $X_0$ following the Markov transition operator $E^\pi$, and $\mu_1$ is the marginal distribution of $X_1$. The composite of operators $\mu(J_K)$ can be expanded in the following infinite sum,
\#\label{eq::inf_sum_mujk}
\mu(J_K) = \sum^\infty_{\ell = 0} \gamma^{\ell}\cdot \mu\Bigl((\EE^{\pi_\kappa})^\ell \circ (E^{\pi^*})^{\kappa+1} \Bigr)-\gamma^{\ell}\cdot \mu\biggl((E^{\pi_\kappa})^\ell \circ \prod^{\kappa}_{j = 0}E^{\pi_j}\biggr),
\#
where each term of the summation is a difference of probability measures multiplied by $\gamma^\ell$. Note that $\|Q\|_{\infty}\leq Q_{\max}$ and $\|\hat Q^\lambda_0\|_{\infty} \leq Q_{\max}$. Therefore, following from \eqref{eq::inf_sum_mujk}, we obtain that
\#\label{eq::err_muJ}
 \mu( J_\kappa|Q^* - \hat Q^\lambda_0|) &\leq\sum^\infty_{\ell = 0} \gamma^{\ell}\cdot \biggl\{  \mu\bigl((E^{\pi_\kappa})^\ell \circ (E^{\pi^*})^{\kappa+1}|Q^* - \hat Q^\lambda_0|\bigr)- \mu\biggl((\EE^{\pi_\kappa})^\ell \circ \prod^{\kappa}_{j = 0}E^{\pi_j}|Q^* - \hat Q^\lambda_0|\biggr)\biggr\}\notag\\
 &\leq \sum^\infty_{\ell = 0} \gamma^{\ell} \cdot 4 Q_{\max} \leq 4Q_{\max}/(1 - \gamma).
\#
Similarly, we bound the term $\mu(F_{i, \kappa}|\epsilon_{i + 1}|)$. Following from the definition of $F_{i, \kappa}$ in \eqref{eq::pf_def_operator_JF}, it holds that
\#\label{eq::err_muF_form}
&\mu(F_{j, \kappa}|\epsilon_{i + 1}|) \notag\\
&\qquad= \sum^\infty_{\ell = 0} \gamma^{\ell}\cdot\biggl\{\mu\bigl( (\EE^{\pi_\kappa})^\ell \circ (E^{\pi^*})^{\kappa -i}|\epsilon_{i + 1}|\bigr) -\mu\biggl((E^{\pi_\kappa})^\ell \circ\prod^{\kappa}_{j = i+1}E^{\pi_j}|\epsilon_{i + 1}|\biggr)\biggr\}.
\#
We first bound the term $\mu( (E^{\pi_\kappa})^\ell \circ (E^{\pi^*})^{\kappa -i}|\epsilon_{i + 1}|)$. We define a Markov process $X_t$ with initial state $X_0 \sim \mu_0$. We define the transition operator to be $E^{\pi^*}$ for $0<t\leq \kappa - i$ and $E^{\pi_\kappa}$ for $\kappa - i < t \leq \kappa - i + \ell$. We then denote by $\tilde\mu$ the marginal distribution of $X_{\kappa - i + \ell}$. Following from the Cauchy-Schwartz inequality, we obtain that
\#\label{eq::err_muF_bound1}
\mu\bigl( (E^{\pi_\kappa})^\ell \circ (E^{\pi^*})^{\kappa -i}|\epsilon_{i + 1}|\bigr) &= \tilde \mu(|\epsilon_{i + 1}|) = \int_{\tilde \cM(\Omega)}|\epsilon_{i + 1}(\omega)| \ud \tilde \mu(\omega)\notag\\
&\leq \biggl(\int_{\tilde \cM(\Omega)} |\epsilon_{i + 1}(\omega)|^2 \ud \nu(\omega)\biggr)^{1/2}\cdot\biggl(\int_{\tilde \cM(\Omega)} \Bigl|\frac{\ud \tilde\mu}{\ud \nu}(\omega)\Bigr|^2 \ud \nu(\omega)\biggr)^{1/2},
\#
where $\ud \tilde\mu/\ud \nu$ is the Radon-Nikodym derivative. Following from Assumption \ref{asu::concen}, it then holds that
\#\label{eq::err_muF_bound2}
\mu\bigl( (E^{\pi_\kappa})^\ell \circ (E^{\pi^*})^{\kappa -i}|\epsilon_{i + 1}|\bigr) \leq \phi(\kappa - i + \ell; \mu, \nu) \cdot \|\epsilon_{i+ 1}\|_\nu.
\#
Note that the same bound holds if we change the transition operators. Therefore, combining \eqref{eq::err_muF_form} and \eqref{eq::err_muF_bound2}, it then holds that
\#\label{eq::err_muF_bound3}
\mu(F_{j, \kappa}|\epsilon_{i + 1}|) \leq \sum^\infty_{\ell = 0}2\gamma^{\ell}\cdot \phi(\kappa - i + \ell; \mu, \nu) \cdot \|\epsilon_{i + 1}\|_\nu.
\#
We denote by $\epsilon_{\max, \kappa} = \max_{i \in [\kappa]} \|\epsilon_{i}\|_\nu$. Following from \eqref{eq::err_muF_bound3}, it then holds that
\#\label{eq::err_muF_bound4}
\sum^{\kappa-1}_{i = 0} \gamma^{\kappa-i} \mu(F_{i, \kappa}|\epsilon_{i + 1}|) &\leq 2 \sum^{\kappa-1}_{i = 0}\sum^\infty_{\ell = 0}\gamma^{\kappa- i + \ell} \cdot\phi(\kappa - i + \ell; \mu, \nu) \cdot \|\epsilon_{i + 1}\|_\nu\notag\\
&\leq2\epsilon_{\max, \kappa}\cdot\sum^\infty_{\ell = 0}\sum^{\ell + \kappa}_{m = \ell + 1}\gamma^{m}\cdot \phi(m; \mu, \nu)\notag\\
&\leq 2\epsilon_{\max, \kappa}\cdot\sum^{\infty}_{m =1}m\cdot\gamma^{m}\cdot \phi(m; \mu, \nu),
\#
Combining \eqref{eq::err_muF_bound4} and Assumption \ref{asu::concen}, we obtain that
\#\label{eq::err_muF}
\sum^{\kappa-1}_{i = 0} \gamma^{\kappa-i} \mu(F_{i, \kappa}|\epsilon_{i + 1}|) &\leq 2\epsilon_{\max, \kappa}\cdot\sum^{\infty}_{m =1}m\cdot\gamma^{m}\cdot \phi(m; \mu, \nu)\leq \frac{2\gamma\cdot\Phi(\mu, \nu)}{(1 -\gamma)^2}\cdot\epsilon_{\max, \kappa}.
\#
Finally, combining \eqref{eq::err_qpik_2}, \eqref{eq::err_muJ}, and \eqref{eq::err_muF}, we conclude that
\#\label{eq::prop_prof_fin}
\|Q^* - Q^{\pi_\kappa}\|_{1, \mu} \leq \frac{2\gamma\cdot\Phi(\mu, \nu)}{(1 -\gamma)^2}\cdot\epsilon_{\max, \kappa} + \frac{4\gamma^{\kappa+1}\cdot Q_{\max}}{1 - \gamma},
\#
which completes the proof of Proposition \ref{thm::error_prop}.

\subsection{Proof of Theorem \ref{thm::one_step}}
\label{pf::one_step}
\begin{proof}
We define $\hat Q = \hat Q^\lambda_k$ and $Q = Q^\lambda_k$ for notational simplicity, where $\hat Q^\lambda_k$ and $Q^\lambda_k$ are define in \eqref{eq::FQI_KRR}. Recall that we denote by $\{\omega_i\}_{i\in[n]}$ the sample that is drawn independently from the sampling distribution $\nu$, where $\omega_i = \delta_{a_i}\times p_{i, s}$, $\hat p_{i, s}$ the empirical approximation of $p_{i, s}$ with $N$ observations, and $\hat\omega_i = \delta_{a_i}\times \hat p_{i, s}$. We further denote by $p_{i, s'}$ the mean-field state after transition, which follows the distribution $P(\cdot\given\omega_{a_i, s_i})$, and $\hat p_{i, s'}$ the empirical approximation of $p_{i, s'}$ with $N$ observations. In what follows, we define $\hat Q^\lambda = \hat Q^\lambda_{k+1}$ and $Q^\lambda = Q^\lambda_{k+1}$ for notational simplicity. More specifically, we define
\#\label{eq::one_step_prob}
Q^{\lambda} = \argmin_{f \in \cH(K)}\frac{1}{n}\sum^n_{i = 1}\bigl(f(\mu_{\hat\omega_i}) - \hat y_i\bigr)^2 + \lambda \|f\|^2_{\cH(K)}, \quad \hat Q^\lambda = \max\{Q^{\lambda}, Q_{\max}\},
\#
where 
\#\label{eq::y_i_hat_def}
\hat y_i = r_i + \gamma \cdot\sup_{a \in \cA} \hat Q(
\mu_{\delta_{a}\times \hat p_{i, s'}}).
\#
We now analyze the one-step approximation error $\|\hat Q^\lambda - T \hat Q\|_{\nu}$, where $T$ is the Bellman optimality operator. Note that by the truncation, it holds that $\hat Q \leq Q_{\max}$ and therefore $T\hat Q \leq Q_{\max}$. Hence, we have
\#\label{eq::freeq1}
\|\hat Q^\lambda - T \hat Q\|_{\nu} \leq \|Q^\lambda - T \hat Q\|_{\nu},
\#
which holds since $\hat Q^\lambda = \max\{Q^\lambda, Q_{\max}\}$. In the sequel, we upper bound the right-hand side of \eqref{eq::freeq1}. To this end, we define
\#\label{eq::regress_true_omega}
\overline Q^{\lambda} = \argmin_{f \in \cH(K)}\frac{1}{n}\sum^n_{i = 1}\bigl(f(\mu_{\omega_i}) - y_i\bigr)^2 + \lambda \|f\|^2_{\cH(K)},
\#
where 
\#\label{eq::y_i_def}
y_i = r_i+ \gamma \cdot\sup_{a \in \cA} \hat Q(\mu_{\delta_{a}\times p_{i, s'}}).
\#
Here $\overline  Q^{\lambda}$ corresponds to the regression in \eqref{eq::one_step_prob} if we use the mean embedding of exact mean-field states $p_{i, s}$ and $p_{i, s'}$ in place of their finite sample approximations $\hat p_{i, s}$ and $\hat p_{i, s'}$, respectively. Meanwhile, for random variables $A_i \in \cA$ and $P_{i, s}\in\cM(\cS)$, we define the random variable $Y_i = r(\delta_{A_i}\times P_{i, s}) + \max_{a\in\cA} Q(\mu_{\delta_a\times P_{i, s'}})$, where $P_{i, s'} \sim P(\cdot\given \delta_A\times P_{i, s})$.  It then holds that $y_i$ is a realization of the random variable $Y_i\given A_i = a, P_{i, s} = p_s$. We denote by $\omega_{a, s} = \delta_a\times p_s$ and $\rho(\cdot\given \omega_{a, s})$ the distribution of $Y_i\given A_i = a, P_{i, s} = p_s$. The function $TQ$ is thus defined as
\$
T Q(\omega_{a,s}) = \EE[Y_i~|~A_i = a, P_{i, s} = p_s] = \EE\Bigl[r(\omega_{a, s}) + \max_{a'\in\cA} Q(\delta_{a'}\times P_s) ~\Bigl|~ P_s \sim P(\cdot ~|~ \omega_{a,s})\Bigr].
\$
Note that under the universality Assumption \ref{asu::bound_kernel}, each state-action configuration $\omega\in\tilde\cM(\Omega)$ uniquely characterizes a mean embedding $\mu_\omega$  \citep{gretton2007kernel, gretton2012kernel}.  For each $\omega \in\tilde\cM(\Omega)$, we define $\rho(\cdot\given\mu_\omega) = \rho(\cdot\given\omega)$. Meanwhile, following from \citep{szabo2015two}, the probability measure $\nu$ on $\tilde\cM(\Omega)$ equivalently defines a probability measure on the space of mean embedings $\cX$. With a slight abuse of notations, we denote by $\nu(\mu_\omega)$ such a measure defined by $\nu(\omega)$, and do not distinguish between them in the sequel. The function $TQ$ then equivalently defines a function on $\cX$ that takes the form,
\$
T Q(\mu_{\omega_i}) = \EE[Y_i\given\mu_{\omega_i}]=\EE_{Y\sim \rho(\cdot\given\mu_{\omega_i})}[Y]. 
\$
Recall that we define for $f \in \cH(K)$ the integral operator $\cC:\cH(K) \mapsto \cH(K)$ as follows,
\$
\cC f(x) = \int_{\tilde\cM(\Omega)} K(x, \mu_\omega) f(\mu_\omega) \ud\nu(\mu_\omega).
\$
We introduce empirical integral operators $\cC_{\omega}$ and $\cC_{\hat\omega}$, which is defined as follows,
\#\label{eq::op_c_def}
\cC_{\omega} f(x) = \frac{1}{n}\sum_{i = 1}^n K(x, \mu_{\omega_i}) f(\mu_{\omega_i}), \quad \cC_{\hat\omega} f(x) = \frac{1}{n}\sum_{i = 1}^n K(x, \mu_{\hat\omega_i}) f(\mu_{\hat\omega_i}).
\#
The operators  $\cC_{\omega}$ and $\cC_{\hat\omega}$ are estimates of $\cC$ with the samples $\{\mu_{\omega_i}\}_{i\in[n]}$ and $\{\mu_{\hat\omega_i}\}_{i\in[n]}$, respectively. The following proposition characterizes the exact form of $Q^\lambda$ and $\overline Q^\lambda$ defined in \eqref{eq::one_step_prob} and \eqref{eq::regress_true_omega}, respectively.
\begin{proposition}[Exact Solutions \citep{caponnetto2007optimal}]
\label{prop::exact_soln}
For $Q^\lambda$ and $\overline Q^\lambda$ defined in \eqref{eq::one_step_prob} and \eqref{eq::regress_true_omega}, respectively, it holds that
\$
Q^\lambda = (\cC_{\hat\omega} + \lambda)^{-1}g_{\hat\omega}, \quad
\overline Q^\lambda = (\cC_{\omega} + \lambda)^{-1}g_{\omega},
\$
where
\#\label{eq::g_def}
g_{\hat\omega}(\cdot) = \frac{1}{n}\sum_{i=1}^n K(\cdot, \mu_{\omega_i})\hat y_i, \quad
g_{\omega}(\cdot) = \frac{1}{n}\sum_{i=1}^n K(\cdot, \mu_{\hat\omega_i})y_i.
\#
\end{proposition}
\begin{proof}
See \cite{caponnetto2007optimal} for a detailed proof.
\end{proof}
In the sequel, we denote by $Q_{\cH, T}$ the projection of $T \hat Q$ onto the RKHS $\cH(K)$ with respect to the norm $\|\cdot\|_\nu$, which is defined as follows,
\$
Q_{\cH, T} \in \argmin_{f\in\cH(K)} \|f - T \hat Q\|_{\nu}.
\$
The following proposition characterizes the exact risk of the one-step approximation.
\begin{proposition}[Exact Risk \citep{caponnetto2007optimal}]
\label{prop::exact_risk}
It holds that 
\#\label{eq::1122}
\| Q^\lambda - T \hat Q\|^2_{\nu}  - \|Q_{\cH, T} - T \hat Q\|^2_{\nu}=
\|\sqrt{\cC}( Q^\lambda -  Q_{\cH, T})\|^2_{\cH(K)} .
\#
\end{proposition}
\begin{proof}
See \cite{caponnetto2007optimal} for a detailed proof.
\end{proof}
Propositions \ref{prop::exact_soln} and \ref{prop::exact_risk} are standard results of kernel ridge regression. Under Assumptions \ref{asu::bound_kernel} and \ref{asu::kernel_approx}, the following lemma adapted from \cite{szabo2015two} upper bounds the right-hand side of \eqref{eq::1122}.
\begin{lemma}[Exact Risk Bound \citep{szabo2015two}]
\label{lem::risk_bound}
Let $0<\eta+\tau < 1$ and $C(\eta) = 32\log^2(6/\eta)$. Under Assumptions \ref{asu::bound_kernel} and \ref{asu::kernel_approx}, for
\#
N \geq 2\varrho\cdot(1 + \sqrt{\log n + \delta})^2\cdot(64L^2\varsigma^2/\lambda^2)^{1/h}, \quad n \geq \frac{2C(\eta)\varsigma\beta b}{(b - 1)\lambda^{1 + 1/b}}, \quad \lambda \leq \|\cC\|_{\cH(K)},
\#
it holds with probability at least $1 - \eta - \tau$ that
\#
\|\sqrt{\cC}( Q^\lambda -  Q_{\cH, T})\|^2_{\cH(K)}
&\leq
\frac{8L^2Q^2_{\max}\bigl(1 + \sqrt{\log (|\cA|\cdot n/2\tau ) }\bigr)^{2h}\cdot (2\varrho)^h}{\lambda \cdot N^h}\cdot\Bigl(1 + \frac{5\varsigma^2}{\lambda^2}\Bigr) \\
&\qquad+ C(\eta) \cdot \Biggl(R \lambda^c +\frac{\varsigma^2 R }{\lambda^{2 - c}n^2} + \frac{\varsigma R \lambda^{c - 1}}{4n} + \frac{\varsigma M^2}{\lambda n^2} + \frac{\Sigma^2 \beta b}{(b - 1)n\lambda^{1/b}}\Biggr).\notag
\#
\end{lemma}
\begin{proof}
See \S\ref{sec::pf_risk_bound} for a detailed proof.
\end{proof}
We define $\psi_T$ as follows,
\#\label{eq::pf_onestep_fin2}
\psi_T= \sup_{k \in [\kappa]}\|T\hat Q^\lambda_k - \Pi_{\cH(K)}(TQ^\lambda_k)\|_{\nu}.
\#
Therefore, combining \eqref{eq::pf_onestep_fin2}, Proposition \ref{prop::exact_risk}, and Lemma \ref{lem::risk_bound}, for
\#
N \geq 2\varrho\cdot(1 + \sqrt{\log n + \delta})^2\cdot(64L^2\varsigma^2/\lambda^2)^{1/h}, \quad n \geq \frac{2C(\eta)\varsigma\beta b}{(b - 1)\lambda^{1 + 1/b}}, \quad \lambda \leq \|\cC\|_{\cH(K)},
\#
it holds with probability at least $1 - \tau - \eta$ that
\#\label{eq::pf_onestep_fin1}
\|Q^\lambda - T \hat Q\|^2_{\nu} 
&\leq
\frac{8L^2Q^2_{\max}\bigl(1 + \sqrt{\log (|\cA|\cdot n/2\tau ) }\bigr)^{2h}\cdot (2\varrho)^h}{\lambda \cdot N^h}\cdot\Bigl(1 + \frac{5\varsigma^2}{\lambda^2}\Bigr) + \psi^2_T\\
&\qquad+ C(\eta) \cdot \Biggl(R \lambda^c +\frac{\varsigma^2 R }{\lambda^{2 - c}n^2} + \frac{\varsigma R \lambda^{c - 1}}{4n} + \frac{\varsigma M^2}{\lambda n^2} + \frac{\Sigma^2 \beta b}{(b - 1)n\lambda^{1/b}}\Biggr),\notag
\#
which concludes the proof of Theorem \ref{thm::one_step}.
\end{proof}

\subsection{Proof of Theorem \ref{thm::err_fin}}
\label{pf::err_fin}
\begin{proof}
Throughout the proof, we denote by $C$ and $C'$ positive absolute constants, whose value may vary from lines to lines. It suffices to study the dominating error of the one-step approximation error in \eqref{eq::one_step}. Recall that that we set $n = N^a$. We split the proof into two cases in terms of $a$.
\vskip4pt
\noindent{\bf{Case 1: }}We first consider the case when $a > h\cdot(1+1/b)/\cdot(c+3)$. Note that the dominating term of $\cG_1$ has the order of $(\log n)^h/(N^h\cdot \lambda^3)$. Thus if $\cG_1$ in \eqref{eq::error_part} converges to zero as $N$ goes to infinity, it holds that $(\log n)^h/(N^h\cdot \lambda^3)$ converges to zero. Therefore, for any positive absolute constant $C$, it holds for a sufficiently large $N$ that $\log n/(N\cdot \lambda^{2/h}) \leq C$. The requirement on $N$ in \eqref{eq::asu_Nnl} thus holds as long as $\cG_1$ converges. We set $\lambda = (\log n/N)^{h/(c+3)} = (a\cdot \log N/N)^{h/(c+3)}$. It then holds that 
\#
n \cdot \lambda^{1 + 1/b} = N^a \cdot (a\cdot \log N/N)^{h\cdot(1 + 1/b)/(c+3)} = \frac{(\log N)^{h\cdot(1 + 1/b)/(c+3)}}{N^{h\cdot(1 + 1/b)/(c+3) - a}}.
\#
Therefore, for $a > h\cdot(1 + 1/b)/(c+3)$, it holds for arbitrary constant $C$ that $n \cdot \lambda^{1 + 1/b}>C$ with a sufficently large $n$. Thus the requirement on $n$ in \eqref{eq::asu_Nnl} holds. Upon computation, it holds that the dominating parts of the right-hand side of \eqref{eq::one_step} are $\cG_1$ and $C(\eta)\cdot R\lambda^c$.  We then obtain that
\#\label{eq::fin_order_match}
\|\hat Q^\lambda_{\kappa +1} - T \hat Q^\lambda_{\kappa}\|_{\nu} \leq \bigl(C + C(\eta)\cdot R\bigr)\cdot \Bigl(\frac{\log(|\cA|\cdot N/\tau)}{N}\Bigr)^{\frac{h\cdot c}{2\cdot(c + 3)}} + \psi_T,
\#
which holds with probability at least $1 - \eta - \tau$. Following from \eqref{eq::fin_order_match} and Theorem \ref{thm::error_prop}, it holds with probability at least $1 - \eta - \tau$ that
\#\label{eq::pf_half}
\|Q^* - Q^{\pi_\kappa}\|_{1, \mu} \leq\frac{2\gamma\cdot\Phi(\mu, \nu)}{(1 -\gamma)^2}\cdot \Bigl(C'\cdot\bigl(\log (|\cA|\cdot N/\tau) / N\bigr)^{\frac{h\cdot c}{2\cdot(c + 3)}} + \text{dist}_T\Bigr)+ \frac{4\gamma^{\kappa+1}\cdot Q_{\max}}{1 - \gamma},
\#
which concludes half of the proof of Theorem \ref{thm::err_fin}. 
\vskip4pt
\noindent{\bf{Case 2: }}We consider the case when $a < h\cdot(1+1/b)/\cdot(c+3)$. We fix $\lambda = 1/N^{ab/(bc + 1)}$. It then holds that
\#\label{eq::pf_fin111}
n \cdot \lambda^{1 + 1/b} = n\cdot(1 / n^{\frac{b}{bc  + 1}})^{\frac{b+1}{b}} = N^{1 - \frac{b + 1}{bc + 1}}.
\#
Following from Assumption \ref{asu::kernel_approx}, it holds that $b > 1$ and $c\geq1$, which implies $(b + 1)/(bc + 1) < 1$. Therefore, following from \eqref{eq::pf_fin111}, it holds for any $C$ that $n \cdot \lambda^{1 + 1/b} \geq C$ for a sufficiently large $n$, and the requirement on $n$ in \eqref{eq::asu_Nnl} holds. Upon computation, for $a < h\cdot(1+1/b)/\cdot(c+3)$, the dominating parts of the right-hand side of \eqref{eq::one_step} are $\lambda^c$ and $1/(n\cdot\lambda^{1 + 1/b})$. Therefore, following from Theorem \ref{thm::one_step}, it holds with probability at least $1 - \eta - \tau$ that
\#\label{eq::fin_order_match_1}
\|\hat Q^\lambda_{\kappa +1} - T \hat Q^\lambda_{\kappa}\|^2_{\nu} \leq C'/N^{\frac{abc}{bc + 1}}.
\#
Then following from Theorems \ref{thm::error_prop} and \eqref{eq::fin_order_match_1}, it holds with probability at least $1 - \eta - \tau$ that
\#
\|Q^* - Q^{\pi_\kappa}\|_{1, \mu} \leq\frac{2\gamma\cdot\Phi(\mu, \nu)}{(1 -\gamma)^2}\cdot \bigl( C'/N^{\frac{abc}{2(bc + 1)}} + \text{dist}_T\bigr)+ \frac{4\gamma^{\kappa+1}\cdot Q_{\max}}{1 - \gamma},
\#
which together with \eqref{eq::pf_half} concludes the proof of Theorem \ref{thm::err_fin}.
\end{proof}

\section{Proof of Lemma \ref{lem::risk_bound}}
\label{sec::pf_risk_bound}
\begin{proof}
The proof strategy is similar to that of Theorem 1 by \cite{caponnetto2007optimal} and Main Theorem by \cite{szabo2015two}. We split the error into the following terms,
\#\label{eq::risk_bound_main}
\|\sqrt{\cC}( Q^\lambda -  Q_{\cH, T})\|^2_{\cH(K)}\leq 2\underbrace{\|\sqrt{\cC}(Q^\lambda - \overline Q^\lambda) \|^2_{\cH(K)} }_{S_1} + 2\underbrace{\|\sqrt{\cC} (\overline Q^\lambda - Q_{\cH, T})\|^2_{\cH(K)}}_{S_2},
\#
where $Q^\lambda$ and $\overline Q^\lambda$ are defined in \eqref{eq::one_step_prob} and \eqref{eq::regress_true_omega}, respectively. In what follows, we establish upper bounds for $S_1$ and $S_2$.
\vskip4pt
\noindent{\bf Bounding $S_1$: } Following from Proposition \ref{prop::exact_soln}, it holds that 
\$
Q^\lambda - \overline Q^\lambda &= (\cC_{\hat\omega} + \lambda)^{-1}g_{\hat\omega} - (\cC_{\omega} + \lambda)^{-1}g_{\omega}\\
&=(\cC_{\hat\omega} + \lambda)^{-1}(g_{\hat\omega} - g_{\omega}) + \bigl((\cC_{\hat\omega} + \lambda)^{-1} -  (\cC_{\omega} + \lambda)^{-1}\bigr)g_{\omega}\\
&=(\cC_{\hat\omega} + \lambda)^{-1}(g_{\hat\omega} - g_{\omega}) + (\cC_{\hat\omega} + \lambda)^{-1}(\cC_{\hat\omega} - \cC_{\omega})(\cC_{\omega} + \lambda)^{-1}g_{\omega}\\
&=(\cC_{\hat\omega} + \lambda)^{-1}\bigl(g_{\hat\omega} - g_{\omega} + (\cC_{\hat\omega} - \cC_{\omega}) Q^\lambda\bigr).
\$
The error term $S_1$ is then upper bounded as follows,
\#\label{eq::risk_bound_full}
S_1 &= \|\sqrt{\cC}( Q^\lambda - \overline Q^\lambda) \|^2_{\cH(K)} \notag\\
&\leq \|\sqrt{\cC}(\cC_{\hat\omega} + \lambda)^{-1}\|^2_{\cH(K)}\cdot\bigl(2\|g_{\hat\omega} - g_{\omega}\|^2_{\cH(K)} + 2\|\cC_{\hat\omega} - \cC_{\omega}\|^2_{\cH(K)}\cdot \|\overline Q^\lambda\|^2_{\cH(K)}\bigr),
\#
where, with a slight abuse of notation, we denote by $\|\cdot\|_{\cH(K)}$ the operator norm when applying to operators defined on $\cH(K)$. We first establish an upper bound for the term $\|\overline Q^\lambda\|_{\cH(K)}^2$. Following from Proposition \ref{prop::exact_soln}, it holds that
\$
\|\overline Q^\lambda\|^2_{\cH(K)} &= \|(\cC_{\omega} + \lambda)^{-1}g_{\omega}\|^2_{\cH(K)}=\|\frac{1}{n}\sum_{i=1}^n (\cC_{\omega} + \lambda)^{-1}K(\cdot, \mu_{\hat\omega_i})y_i\|_{\cH(K)}^2\notag\\
&\leq \|(\cC_{\omega} + \lambda)^{-1}\|^2_{\cH(K)}\cdot \frac{1}{n}\sum^n_{i = 1}\|K(\cdot, \mu_{\hat\omega_i})\|_{\cH(K)}^2\cdot|y_i|^2\notag\\
&=  \|(\cC_{\omega} + \lambda)^{-1}\|^2_{\cH(K)} \cdot \frac{1}{n}\sum^n_{i = 1} K(\mu_{\hat\omega_i}, \mu_{\hat\omega_i})\cdot|y_i|^2,
\$
where the last inequality follows from the definition of operator norm and the Cauchy-Schwartz inequality, and the last equality follows from the reproducing property of kernel $K(\cdot, \cdot)$. Note that $\|(\cC_{\omega} + \lambda)^{-1}\|^2_{\cH(K)} \leq 1/\lambda$. Therefore, following from Assumption \ref{asu::bound_kernel}, it holds that
\#\label{eq::risk_bound_part_4}
\|\overline Q^\lambda\|^2_{\cH(K)} \leq (1/\lambda)^2 \cdot \frac{1}{n}\sum^n_{i = 1}\varsigma Q_{\max}^2 \leq \varsigma Q^2_{\max}/\lambda^2.
\#
We now bound the remaining terms in $\cS_1$. The following lemmas hold, which establish the upper bounds for terms in the right-hand side of \eqref{eq::risk_bound_full}.
\begin{lemma}[Concentration of Empirical Mean \citep{altun2006unifying}]
\label{lem::empirical_mean_concen}
Let $\epsilon$ be a positive absolute constant. Let $\hat\omega$ be the empirical approximation of $\omega \in \tilde\cM(\Omega)$, which takes the form
\$
\hat\omega = \frac{1}{N}\sum^N_{i = 1}\delta_{\omega_i},
\$
where $\{\omega_i\}^N_{i = 1}$ are independent samples of $\omega$. Under Assumption \ref{asu::bound_kernel}, it holds with probability at least $1 - \exp(-2\epsilon)$ that
\$
\|\mu_{\hat\omega} - \mu_\omega\|_{\cH(k)} \leq (1 + \sqrt{\epsilon})\cdot\sqrt{2\varrho/N}.
\$
\end{lemma}
\begin{proof}
See \S\ref{sec::pf_sup_lem1} for a detailed proof.
\end{proof}
\begin{lemma}
\label{lem::risk_bound_part1}
Let $\epsilon$ be a positive absolute constant. It holds with probability at least $1 - |\cA|\cdot n \exp(-2\epsilon)$ that
\#\label{eq::risk_bound_part1}
\|g_{\hat\omega} - g_{\omega}\|^2_{\cH(K)} \leq ( 1+\varsigma/\lambda^2 )\cdot L^2Q^2_{\max}\cdot (1 + \sqrt{\epsilon})^{2h} \cdot(2\varrho/N)^h.
\#
\end{lemma}
\begin{proof}
See \S\ref{sec::pf_risk_bound_part1} for a detailed proof.
\end{proof}

\begin{lemma}
\label{lem::risk_bound_part2}
Let $\epsilon$ be a positive absolute constant. Under Assumption \ref{asu::bound_kernel}, it holds with probability at least $1 - n\exp(-2\epsilon)$ that
\#\label{eq::risk_bound_part2}
\|\cC_{\omega} - \cC_{\hat\omega}\|^2_{\cH(K)}\leq 4L^2\varsigma\cdot(1 + \sqrt{\epsilon})^{2h}\cdot(2\varrho/N)^h.
\#
\end{lemma}
\begin{proof}
See \S\ref{sec::pf_risk_bound_part2} for a detailed proof.
\end{proof}

\begin{lemma}
\label{lem::risk_bound_part3}
Let $\epsilon$ and $\eta$ be positive absolute constants. Under Assumption \ref{asu::bound_kernel}, for
\#\label{eq::part3_assum}
\lambda \leq \|\cC\|_{\cH(K)}, \quad N \geq 2\varrho\cdot(1 + \sqrt{\epsilon})^2\cdot(64L^2\varsigma/\lambda^2)^{1/h},
\# 
it holds with probability at least $1 - \eta/3 - n\exp(-2\epsilon)$ that
\#\label{eq::risk_bound_part3}
\|\sqrt{\cC}(\cC_{\hat\omega} + \lambda)^{-1}\|_{\cH(K)}\leq 2/\sqrt{\lambda}.
\#
\end{lemma}
\begin{proof}
See \S\ref{sec::pf_risk_bound_part3} for a detailed proof.
\end{proof}
Finally, combining Lemma \ref{lem::risk_bound_part1}, \ref{lem::risk_bound_part2}, \ref{lem::risk_bound_part3}, and \eqref{eq::risk_bound_part_4}, it follows from \eqref{eq::risk_bound_full} that
\#\label{eq::risk_bound_S1}
S_1 &\leq \|\sqrt{\cC}(\cC_{\hat\omega} + \lambda)^{-1}\|^2_{\cH(K)}\cdot\bigl(2\|g_{\hat\omega} - g_{\omega}\|^2_{\cH(K)} + 2\|\cC_{\hat\omega} - \cC_{\omega}\|^2_{\cH(K)}\cdot \|\overline Q^\lambda\|^2_{\cH(K)}\bigr)\notag\\
&\leq\frac{8L^2Q^2_{\max}\bigl(1 + \sqrt{\epsilon }\bigr)^{2h}\cdot (2\varrho)^h}{\lambda \cdot N^h}\cdot\Bigl(1 + \frac{5\varsigma^2}{\lambda^2}\Bigr),
\#
which holds with probability at least $1 -\eta/3- n\cdot|\cA|\cdot\exp(-2\epsilon)$.

\vskip4pt
\noindent{\bf Bounding $S_2$: }Note that the error $S_2$ is the excess risk of a standard kernel ridge regression with $n$ samples. Following from Assumptions \ref{asu::bound_kernel} and \ref{asu::kernel_approx}, it holds that 
\$
|y - Q_{\cH, T}(\mu_\omega)| \leq |y| + |Q_{\cH, T}(\mu_\omega)|\leq Q_{\max} + \sqrt{\varsigma}\cdot\max\{1, \|\cC\|_{\cH(K)}\}\cdot R.
\$
Therefore, it holds for some positive absolute constants $\Sigma$ and $M$ that
\#
\int_{\RR} \exp\bigl(|y - Q_{\cH, T}(\mu_\omega)|/ M\bigr) - \frac{|y - Q_{\cH, T}(\mu_\omega)|}{M} - 1 \ud \rho(y|\mu_\omega) \leq \frac{\Sigma^2}{2M^2},
\#
which holds for $\nu$-almost surely all $\omega \in \tilde\cM(\Omega)$. The following lemma follows from Theorem 1 by \cite{caponnetto2007optimal}, which characterizes an upper bound of $\cS_2$. 
\begin{lemma}[Excess Risk of Kernel Ridge Regression \cite{caponnetto2007optimal}]
\label{lem::risk_KRR}
Let $0<\eta<1$, $\lambda >0$, and $n \in \NN$. Let $C(\eta) = 32\log^2(6/\eta)$ be a positive absolute constant. For
\$
n \geq \frac{2C(\eta)\varsigma\beta b}{(b - 1)\lambda^{1 + 1/b}}, \quad \lambda \leq \|\cC\|_{\cH(K)},
\$
under Assumptions \ref{asu::bound_kernel} and \ref{asu::kernel_approx}, it holds with probability at least $1 - \eta$ that
\#\label{eq::bound_S2}
&S_2 = \|\sqrt{\cC} (\overline Q^\lambda - Q_{\cH, T})\|^2_{\cH(K)} \leq C(\eta) \cdot \Biggl(R \lambda^c  + \frac{\varsigma^2 R }{\lambda^{2 - c}n^2} + \frac{\varsigma R \lambda^{c - 1}}{4n} + \frac{\varsigma M^2}{\lambda n^2} + \frac{\Sigma^2 \beta b}{(b - 1)n\lambda^{1/b}}\Biggr).
\#
\end{lemma}
\begin{proof}
See \cite{caponnetto2007optimal} for a detailed proof.
\end{proof}
We fix the constant $2\epsilon = \log (|\cA|\cdot n/\tau)$ in \eqref{eq::risk_bound_S1}. The probability that the error bound in \eqref{eq::risk_bound_S1} holds then becomes at least $1 -\eta/3- \tau$. By combining \eqref{eq::risk_bound_S1} and Lemma \ref{lem::risk_KRR}, it follows from \eqref{eq::risk_bound_main} that
\#\label{eq::risk_bound_conclusion}
\|\sqrt{\cC}(Q^\lambda -  Q_{\cH, T})\|^2_{\cH(K)}
&\leq
\frac{8L^2Q^2_{\max}\bigl(1 + \sqrt{\log (|\cA|\cdot n/\tau )}\bigr)^{2h}\cdot (2\varrho)^h}{\lambda \cdot N^h}\cdot\Bigl(1 + \frac{5\varsigma^2}{\lambda^2}\Bigr) \\
&\qquad+ C(\eta) \cdot \Biggl(R \lambda^c +\frac{\varsigma^2 R }{\lambda^{2 - c}n^2} + \frac{\varsigma R \lambda^{c - 1}}{4n} + \frac{\varsigma M^2}{\lambda n^2} + \frac{\Sigma^2 \beta b}{(b - 1)n\lambda^{1/b}}\Biggr),\notag
\#
which holds with probability at least $1 - \eta - \tau$ and thus concludes the proof of Lemma \ref{lem::risk_bound}.
\end{proof}

\section{Proof of Supporting Lemmas}
\subsection{Proof of Lemma \ref{lem::empirical_mean_concen}}
\label{sec::pf_sup_lem1}
\begin{proof}
Lemma \ref{lem::empirical_mean_concen} is a corollary of Theorem 15 by \cite{altun2006unifying}. We first show that $\mu_\omega$ is bounded for any $\omega\in\tilde\cM(\Omega)$. It holds from the reproducing property of the kernel $k(\cdot, \cdot)$ that
\#\label{eq::pf_emc_eq1}
|k(x, x')| &= |\langle k(\cdot, x), k(\cdot, x') \rangle_{\cH(k)}|\notag\\
&\leq \bigl(\langle k(\cdot, x), k(\cdot, x) \rangle_{\cH(k)}\cdot\langle k(\cdot, x'), k(\cdot, x') \rangle_{\cH(k)}\bigr)^{1/2}=\sqrt{k(x, x)\cdot k(x', x')}.
\#
Therefore, upon taking the expectation of \eqref{eq::pf_emc_eq1}, it follows from Assumption \ref{asu::bound_kernel} that
\#\label{eq::pf_emc_eq2}
\|\mu_\omega\|^2_{\cH(k)} = \EE_{x\sim\omega, x'\sim\omega}\bigl[k(x, x')\bigr]\leq\EE_{x\sim\omega, x'\sim\omega}\bigl[\sqrt{k(x, x)\cdot k(x', x')}\bigr] \leq \varrho,
\#
where the last inequality follows from Assumption \ref{asu::bound_kernel}. It then suffices to compute for $\omega\in\tilde\cM(\Omega)$ the following risk,
\$
R_N = N^{-1/2}\Bigl(\EE_{x\sim\omega, x'\sim\omega}\bigl[k(x, x) - k(x, x')\bigr]\Bigr)^{1/2}.
\$
Combining \eqref{eq::pf_emc_eq2} and Assumption \ref{asu::bound_kernel}, it holds that
\#\label{eq::pf_suplem1_th15}
R_N &\leq N^{-1/2}\Bigl(\EE_{x\sim\omega, x'\sim\omega}\bigl[|k(x, x)| + |k(x, x')|\bigr]\Bigr)^{1/2}\leq N^{-1/2}\Bigl(\EE_{x\sim\omega, x'\sim\omega}\bigl[|k(x, x)| + |k(x, x')|\bigr]\Bigr)^{1/2}\notag\\
       &\leq N^{-1/2}\Bigl(\EE_{x\sim\omega}\bigl[|k(x, x)|\bigr] + \EE_{x\sim\omega, x'\sim\omega}\bigl[\sqrt{k(x, x)\cdot k(x', x')}\bigr]\Bigr)^{1/2}\leq \sqrt{2\varrho/N}.
\#
Following from \eqref{eq::pf_suplem1_th15} and Theorem 15 by \cite{altun2006unifying}, it holds for positive absolute constant $\eta_1$ that
\$
\|\mu_{\hat\omega} - \mu_\omega\|_{\cH(k)}\leq R_N + \eta_1 \leq \sqrt{2\varrho/N} + \eta_1,
\$
which holds with probability $1 - \exp(-\eta^2_1N/\varrho)$. Thus, we complete the proof of Lemma \ref{lem::empirical_mean_concen} by setting $\eta_1 = \sqrt{2\epsilon \varrho/N}$. 
\end{proof}
\subsection{Proof of Lemma \ref{lem::risk_bound_part1}}
\label{sec::pf_risk_bound_part1}
\begin{proof}
Following from the definition of $g_{\omega}$ and $g_{\hat\omega}$ in \eqref{eq::g_def}, it holds that
\#\label{eq::part1_pf_eq1}
\|g_{\hat\omega} - g_{\omega}\|^2_{\cH(K)} &= \|\frac{1}{n}\sum_{i=1}^n \bigl(K(\cdot, \mu_{\hat\omega_i})\cdot\hat y_i- K(\cdot, \mu_{\omega_i})\cdot y_i\bigr) \|^2_{\cH(K)}\notag\\
&\leq \frac{1}{n}\sum_{i=1}^n\|K(\cdot, \mu_{\hat\omega_i})- K(\cdot, \mu_{\omega_i})\|^2_{\cH(K)}\cdot|y_i|^2 +  \frac{1}{n}\sum_{i=1}^n\|K(\cdot, \mu_{\hat\omega_i})\|^2_{\cH(K)}\cdot|y_i - \hat y_i|^2.
\#
Combining \eqref{eq::part1_pf_eq1} and Assumption \ref{asu::bound_kernel}, we obtain that
\#\label{eq::part1_pf_eq2}
\frac{1}{n}\sum_{i=1}^n\|K(\cdot, \mu_{\hat\omega_i})- K(\cdot, \mu_{\omega_i})\|^2_{\cH(K)}\cdot|y_i|^2  \leq \frac{L^2Q^2_{\max}}{n}\sum_{i=1}^n\|\mu_{\hat\omega_i} - \mu_{\omega_i}\|^{2h}_{\cH(k)}.
\#
Meanwhile, following from the definition of $\hat y_i$ and $y_i$ in \eqref{eq::y_i_hat_def} and \eqref{eq::y_i_def}, it holds that
\#\label{eq::part1_pf_yi1}
|y_i - \hat y_i|^2 &\leq |\sup_{a \in \cA} \hat Q(\delta_{a}\times \hat p'_{i, s}) - \sup_{a \in \cA} \hat Q(\delta_{a}\times p'_{i, s})|^2\notag\\
&\leq \sup_{a\in\cA} \bigl|\min\bigl\{Q (\delta_{a}\times \hat p'_{i, s}), Q_{\max}\bigr\} - \min\bigl\{Q(\delta_{a}\times p'_{i, s}), Q_{\max}\bigr\}\bigr|^2\notag\\
&\leq \sup_{a\in\cA} |Q(\delta_{a}\times \hat p'_{i, s}) - Q(\delta_{a}\times p'_{i, s})|^2,
\#
where recall that we denote by $\hat Q = \{Q, Q_{\max}\}$ the previous update of action-value function and $Q \in \cH(K)$ the output of the associated regression, and the last inequality follows from the fact that the function $\min\{\cdot, Q_{\max}\}$ is a contraction mapping. Note that for any $a\in\cA$ and $\hat p'_{i, s}, p'_{i, s} \in \cM(\cS)$, it holds that
\#\label{eq::part1_pf_yi2}
|Q(\delta_{a}\times \hat p'_{i, s}) -Q(\delta_{a}\times p'_{i, s})|^2 &= \langle K(\cdot, \mu_{\omega'_i}) - K(\cdot, \mu_{\hat\omega'_i}), Q\rangle^2_{\cH(K)}\notag\\
&\leq\|Q\|^2_{\cH(K)}\cdot \|K(\cdot, \mu_{\omega'_i}) - K(\cdot, \mu_{\hat\omega'_i})\|^2_\cH(K),
\#
where we denote by $\omega'_i = \delta_{a}\times p'_{i, s}$ and $\hat\omega'_i = \delta_{a}\times \hat p'_{i, s}$. Following from \eqref{eq::risk_bound_part_4} and Assumption \ref{asu::bound_kernel}, it then holds that
\#\label{eq::part1_pf_yi3}
|Q(\delta_{a}\times \hat p'_{i, s}) -Q(\delta_{a}\times p'_{i, s})|^2 \leq \varsigma L^2 Q_{\max}^2/\lambda^2\cdot\|\mu_{\omega'_i} - \mu_{\hat\omega'_i}\|^{2h}_{\cH(k)}.
\#
Combining \eqref{eq::part1_pf_yi1}, \eqref{eq::part1_pf_yi2}, \eqref{eq::part1_pf_yi3}, and Assumption \ref{asu::bound_kernel}, it then holds that
\#\label{eq::part1_pf_yi4}
\frac{1}{n}\sum_{i=1}^n\|K(\cdot, \mu_{\hat\omega_i})\|^2_{\cH(K)}\cdot|y_i - \hat y_i|^2 \leq \frac{B_K^2 L^2 Q_{\max}^2}{\lambda^2\cdot n}\sum_{i=1}^n \max_{a \in \cA} \|\mu_{\omega'_i} - \mu_{\hat\omega'_i}\|^{2h}_{\cH(k)}.
\#
By a union bound argument together with Lemma \ref{lem::empirical_mean_concen}, it follows from \eqref{eq::part1_pf_eq2} and \eqref{eq::part1_pf_yi4} that
\#
\|g_{\hat\omega} - g_{\omega}\|^2_{\cH(K)} &\leq ( 1+\varsigma/\lambda^2 )\cdot L^2Q^2_{\max}/n \cdot \sum_{i=1}^n(1 + \sqrt{\epsilon})^{2h} \cdot(2\varrho/N)^h \notag\\
&\leq ( 1+\varsigma/\lambda^2 )\cdot L^2Q^2_{\max}\cdot (1 + \sqrt{\epsilon})^{2h} \cdot(2\varrho/N)^h,
\#
which holds with probability at least $1 - n\cdot|\cA|\cdot\exp(-2\epsilon)$. Thus, we complete the proof of Lemma \ref{lem::risk_bound_part1}. 
\end{proof}

\subsection{Proof of Lemma \ref{lem::risk_bound_part2}}
\label{sec::pf_risk_bound_part2}
\begin{proof}
We fix an arbitrary $f \in\cH(K)$. Following from the definition of $\cC_{\omega}$ and $\cC_{\hat\omega}$ in \eqref{eq::op_c_def}, it holds that
\#\label{eq::part2_pf_eq1}
\|\cC_{\omega}f - \cC_{\hat\omega}f\|^2_{\cH(K)} &= \|\frac{1}{n}\sum^n_{i = 1}K(\cdot, \mu_{\omega_i})f(\mu_{\omega_i}) - K(\cdot, \mu_{\hat\omega_i})f(\mu_{\hat\omega_i})\|^2_{\cH(K)}\notag\\
&\leq\frac{2}{n}\sum^{n}_{i = 1}\|\bigl(K(\cdot, \mu_{\omega_i})-K(\cdot, \mu_{\hat\omega_i})\bigr)f(\mu_{\omega_i})\|^2_{\cH(K)} \notag\\
&\qquad+ \frac{2}{n}\sum^n_{i=1}\|K(\cdot, \mu_{\hat\omega_i})\bigl(f(\mu_{\omega_i}) - f(\mu_{\hat\omega_i})\bigr)\|^2_{\cH(K)}.
\#
In what follows, we establish upper bounds for each terms in the right-hand side of \eqref{eq::part2_pf_eq1}. Following from the reproducing property of the kernel $K(\cdot, \cdot)$, it holds that
\#\label{eq::part2_pf_eq2}
|f(\mu_{\omega})|^2 &= |\langle f(\cdot), K(\cdot, \mu_\omega)\rangle_{\cH(K)}|^2\leq \|f\|^2_{\cH(K)}\cdot\|K(\cdot, \mu_\omega)\|^2_{\cH(K)} = \|f\|^2_{\cH(K)}\cdot K(\mu_\omega, \mu_\omega).
\#
Combining \eqref{eq::part2_pf_eq2} and Assumption \ref{asu::bound_kernel}, we obtain that
\#\label{eq::part2_pf_eq3}
\|\bigl(K(\cdot, \mu_{\omega_i})-K(\cdot, \mu_{\hat\omega_i})\bigr)f(\mu_{\omega_i})\|^2_{\cH(K)}&\leq\|K(\cdot, \mu_{\omega_i})-K(\cdot, \mu_{\hat\omega_i})\|^2_{\cH(K)} \cdot \varsigma\cdot\|f\|^2_{\cH(K)}\notag\\
&\leq L^2\varsigma\cdot\|\mu_{\omega_i} - \mu_{\hat\omega_i}\|^{2h}_{\cH(k)}\cdot\|f\|^2_{\cH(K)}.
\#
Similarly, following from the reproducing property of the kernel $K(\cdot, \cdot)$, it holds that
\#\label{eq::part2_pf_eq4}
|f(\mu_{\omega}) - f(\mu_{\hat\omega})|^2 &= |\langle f(\cdot), K(\cdot, \mu_\omega) - K(\cdot, \mu_{\hat\omega})\rangle_{\cH(K)}|^2\leq \|f\|^2_{\cH(K)}\cdot\|K(\cdot, \mu_\omega) - K(\cdot, \mu_{\hat\omega})\|^2_{\cH(K)}.
\#
Combining \eqref{eq::part2_pf_eq4} and Assumption \ref{asu::bound_kernel}, we obtain that
\#\label{eq::part2_pf_eq5}
\|K(\cdot, \mu_{\hat\omega_i})\bigl(f(\mu_{\omega_i}) - f(\mu_{\hat\omega_i})\bigr)\|^2_{\cH(K)} &\leq\|K(\cdot, \mu_{\hat\omega_i})\|_{\cH(K)}^2\cdot \|f\|^2_{\cH(K)}\cdot\|K(\cdot, \mu_{\omega_i}) - K(\cdot, \mu_{\hat\omega_i})\|^2_{\cH(K)}\notag\\
&\leq L^2\varsigma \cdot\|\mu_{\omega_i} - \mu_{\hat\omega_i}\|^{2h}_{\cH(k)}\cdot\|f\|^2_{\cH(K)}.
\# 
Finally, combining \eqref{eq::part2_pf_eq1}, \eqref{eq::part2_pf_eq3}, and \eqref{eq::part2_pf_eq5}, we obtain that
\$
\|\cC_{\omega}f - \cC_{\hat\omega}f\|^2_{\cH(K)}\leq \frac{4L^2\varsigma}{n}\cdot\sum^n_{i = 1}\|\mu_{\omega_i} - \mu_{\hat\omega_i}\|^{2h}_{\cH(k)}\cdot\|f\|^2_{\cH(K)},
\$
which is equivalent with the followings,
\#\label{eq::part2_pf_eq6}
\|\cC_{\omega}- \cC_{\hat\omega}\|^2_{\cH(K)}\leq\frac{4L^2\varsigma}{n}\cdot\sum^n_{i = 1}\|\mu_{\omega_i} - \mu_{\hat\omega_i}\|^{2h}_{\cH(k)}.
\#
Following from \eqref{eq::part2_pf_eq6} and Lemma \ref{lem::empirical_mean_concen} together with a union bound argument, it then holds for any positive absolute constant $\epsilon$ that
\#
\|\cC_{\omega}- \cC_{\hat\omega}\|^2_{\cH(K)}\leq4L^2\varsigma\cdot(1 + \sqrt{\epsilon})^{2h}\cdot(2\varrho/N)^h,
\#
which holds with probability at least $1 - n\exp(-2\epsilon)$. Thus, we complete the proof of Lemma \ref{lem::risk_bound_part2}. 
\end{proof}

\subsection{Proof of Lemma \ref{lem::risk_bound_part3}}
\label{sec::pf_risk_bound_part3}
\begin{proof}
It holds that
\$
\sqrt{\cC}(\cC_{\hat\omega} + \lambda)^{-1} = \sqrt{\cC}(\cC + \lambda)^{-1}(I - (\cC - \cC_{\hat\omega})(\cC + \lambda)^{-1})^{-1}.
\$
Therefore, following from the Neumann sequence, it holds that
\$
\|\sqrt{\cC}(\cC_{\hat\omega} + \lambda)^{-1}\|_{\cH(K)} \leq \|\sqrt{\cC}(\cC + \lambda)^{-1}\|_{\cH(K)}\sum^\infty_{n = 0}\|(\cC - \cC_{\hat\omega})(\cC + \lambda)^{-1}\|^n_{\cH(K)}.
\$
Following form \eqref{eq::part3_assum}, it holds that $\|\sqrt{\cC}(\cC + \lambda)^{-1}\|_{\cH(K)} \leq 1/(2\sqrt{\lambda})$. It then suffices to upper bound the norm $\|(\cC - \cC_{\hat\omega})(\cC + \lambda)^{-1}\|_{\cH(K)}$. Note that
\#\label{eq::part3_pf_eq0}
&\|(\cC - \cC_{\hat\omega})(\cC + \lambda)^{-1}\|_{\cH(K)}\leq\|(\cC - \cC_{\omega})(\cC + \lambda)^{-1}\|_{\cH(K)} + \|(\cC_{\omega} - \cC_{\hat\omega})(\cC + \lambda)^{-1}\|_{\cH(K)}.
\#
Following from \cite{caponnetto2007optimal}, it holds with probability at least $1 - \eta/3$ that
\#\label{eq::part3_pf_eq1}
\|(\cC - \cC_{\omega})(\cC + \lambda)^{-1}\|_{\cH(K)} \leq 1/2.
\#
Meanwhile, following from Lemma \ref{lem::risk_bound_part2}, it holds with probability at least $1 - n\exp(-2\epsilon)$ that
\#\label{eq::part3_pf_eq2}
\|\cC_{\omega} - \cC_{\hat\omega}\|^2_{\cH(K)}\leq 4L^2\varsigma\cdot(1 + \sqrt{\epsilon})^{2h}\cdot(2\varrho/N)^h.
\#
Note that we set $N$ to be such that
\#\label{eq::part3_pf_eq3}
N \geq 2\varrho\cdot(1 + \sqrt{\epsilon})^2\cdot(64L^2\varsigma/\lambda^2)^{1/h}.
\#
Meanwhile, it holds that $\|(\cC + \lambda)^{-1}\|_{\cH(K)} \leq 1/\lambda$. Therefore, following from \eqref{eq::part3_pf_eq2} and \eqref{eq::part3_pf_eq3}, it holds with probability at least $1 - n\exp(-2\epsilon)$ that
\#\label{eq::part3_pf_eq4}
\|(\cC_{\omega} - \cC_{\hat\omega})(\cC + \lambda)^{-1}\|_{\cH(K)} \leq \|\cC_{\omega} - \cC_{\hat\omega}\|_{\cH(K)}\cdot\|(\cC + \lambda)^{-1}\|_{\cH(K)} \leq \lambda/4 \cdot 1/\lambda = 1/4,
\#
Finally, combining \eqref{eq::part3_pf_eq0}, \eqref{eq::part3_pf_eq1}, and \eqref{eq::part3_pf_eq4}, it holds with probability at least $1 - \eta/3 - n\exp(-2\epsilon)$ that
\#
\|(\cC - \cC_{\hat\omega})(\cC + \lambda)^{-1}\|_{\cH(K)} &\leq \frac{1}{2\sqrt{\lambda}}\cdot \sum^\infty_{n = 0}(3/4)^n = 2/\sqrt{\lambda},
\#
which concludes the proof of Lemma \ref{lem::risk_bound_part3}.
\end{proof}

\end{document}